\documentclass[leqno,11pt]{amsart}
\usepackage[left=1in,right=1in,bottom=1in,top=1in]{geometry}
\usepackage{lipsum}
\usepackage{mathrsfs}
\usepackage{bm}
\usepackage{scalerel}[2016/12/29]
\usepackage{changepage}

\usepackage{hyphenat}
\usepackage{subcaption}
\usepackage{graphicx}
\usepackage[space]{grffile}
\usepackage{hyperref}
\hypersetup{
    colorlinks,
    linkcolor={red!50!black},
    citecolor={blue!50!black},
    urlcolor={blue!80!black}
}
\usepackage{bbold}
\usepackage{tikz}
\usepackage{pgfplots}
\pgfplotsset{compat=1.17}
\usepackage[space]{grffile}
\usetikzlibrary{calc,angles,positioning,shapes,intersections,quotes,decorations.markings,patterns,arrows,decorations.pathreplacing,calligraphy}
\tikzset{
	position label/.style={
		below = 4pt,
		text height = 2ex,
		text depth = 1.5ex
	},
	brace/.style={
		decoration={brace, mirror},
		decorate
	}
}

\usepackage{longtable}
\usepackage{lscape}

\usepackage{array,  caption}
\usepackage{ragged2e}
\usepackage{multirow}
\usepackage{hhline}
\usepackage{makecell}
\usepackage{centernot}
\usepackage{enumitem}
\usepackage{bbm}
\usepackage{soul}
\usepackage{mathtools}
\usepackage{commath}
\usepackage{pifont}

\usepackage{setspace}

\allowdisplaybreaks

\usepackage[normalem]{ulem}
\usepackage{relsize}

\usepackage[sans]{dsfont}
\usepackage[bottom]{footmisc}

\usepackage{natbib}
\setcitestyle{authoryear,comma}

\definecolor{cyan}{cmyk}{1, 0.4, 0, 0}

\usepackage[normalem]{ulem}
\usepackage{relsize}
\usepackage{commath}
\usepackage{mathtools}
\allowdisplaybreaks

\title{Tangential Wasserstein Projections}
\author{Florian Gunsilius, Meng Hsuan Hsieh, and Myung Jin Lee}
\date{\today}
\thanks{Email correspondences: \texttt{\{ffg,rexhsieh,leemjin\}@umich.edu}. FG is supported by a MITRE research award. MH gratefully acknowledges financial support from the Ross School of Business. This research was supported in part through computational resources and services provided by Advanced Research Computing (ARC), a division of Information and Technology Services (ITS) at the University of Michigan, Ann Arbor. The authors thank Sinho Chewi, Thibaut Le Gouic, and Philippe Rigollet for helpful discussions and comments. All errors are the authors'.}

\usepackage{makecell}
\usepackage{booktabs}

\usepackage[capitalize,nameinlink,noabbrev]{cleveref} 

\usepackage{float}
\usepackage{rotating}

\usepackage{color}

\usepackage{amsthm}

\usepackage[utf8]{inputenc}
\usepackage[T1]{fontenc}
\usepackage[english]{babel}
\usepackage{textcomp}

\usepackage{amssymb} 
\usepackage{amsmath} 
\usepackage{amsfonts}

\makeatletter
\renewcommand*\env@matrix[1][*\c@MaxMatrixCols c]{%
	\hskip -\arraycolsep
	\let\@ifnextchar\new@ifnextchar
	\array{#1}}
\makeatother
\usepackage{dcolumn}

\newcolumntype{C}[1]{>{\arraybackslash}p{#1}}

\newtheorem{definition}{Definition}[section]
\newtheorem{proposition}{Proposition}[section]

\newtheorem{corollary}[definition]{Corollary}

\theoremstyle{definition}

\usepackage{stmaryrd}

\DeclareMathOperator{\R}{\mathbb{R}}

\DeclareMathOperator{\prob}{\mathds{P}}

\DeclareMathOperator{\Normal}{\mathcal{N}}

\DeclareMathOperator{\convexhull}{\mathfrak{Co}}

\DeclareMathOperator{\grad}{\nabla}

\DeclareMathOperator{\Chi}{\mathcal{X}}
\DeclareMathOperator{\YY}{\mathcal{Y}}
\DeclareMathOperator{\PP}{\mathscr{P}}
\DeclareMathOperator{\id}{\mathrm{Id}}
\DeclareMathOperator{\tangent}{\mathcal{T}}

\DeclareMathOperator*{\argmin}{\arg\!\min}

\DeclareMathOperator{\wasserstein}{\mathcal{W}}

\numberwithin{equation}{section}

\renewcommand{\subset}{\subseteq}
\renewcommand{\leq}{\leqslant}
\renewcommand{\geq}{\geqslant}
\renewcommand{\epsilon}{\varepsilon}
\renewcommand{\phi}{\varphi}

\newcommand{\condset}[4]{\left\{ #1  : \: #2 #3 #4 \right\}}

\usepackage{bbm}

\begin{document}

\begin{abstract}
    We develop a notion of projections between sets of probability measures using the geometric properties of the $2$-Wasserstein space. It is designed for general multivariate probability measures, is computationally efficient to implement, and provides a unique solution in regular settings. The idea is to work on regular tangent cones of the Wasserstein space using generalized geodesics. Its structure and computational properties make the method applicable in a variety of settings, from causal inference to the analysis of object data. An application to estimating causal effects yields a generalization of the notion of synthetic controls to multivariate data with individual-level heterogeneity, as well as a way to estimate optimal weights jointly over all time periods.
\end{abstract}

\maketitle


\section{Introduction}
The concept of projections, that is, approximating a target quantity of interest by an optimally weighted combination of other quantities, is of fundamental relevance in statistics. Projections are generally defined between random variables in appropriately defined linear spaces \citep[e.g.][chapter 11]{van2000asymptotic}. In modern statistics and machine learning applications, the objects of interest are often probability measures themselves. Examples range from object- and functional data \citep[e.g.][]{marron2014overview} to causal inference with individual heterogeneity \citep[e.g.][]{athey2015machine}. 

We introduce a notion of projection between sets of probability measures supported on Euclidean spaces. The proposed definition is applicable between sets of general probability measures with different supports and possesses good computational and statistical properties. It also provides a unique solution to the projection problem under mild conditions and can replicate the geometric properties of the target measure, such as its shape and support. To achieve this, we work in the $2$-Wasserstein space, that is, the set of all probability measures with finite second moments equipped with the $2$-Wasserstein distance \citep{villani_optimal_2009}. 

Importantly, we focus on the multivariate setting, i.e.~we consider the Wasserstein space over some Euclidean space $\mathbb{R}^d$, denoted by $\wasserstein_2$, where the dimension $d$ can be high. The multivariate setting poses particular challenges from a mathematical, computational, and statistical perspective. In particular, $\wasserstein_2$ is a positively curved metric space for $d>1$ \citep[e.g.][]{ambrosio_gradient_2008, kloeckner2010geometric}, which requires us to develop a definition of projection on positively curved metric spaces. Moreover, the $2$-Wasserstein distance between two probability measures is defined as the value function of the Monge-Kantorovich optimal transportation problem \citep[chapter 2]{villani_topics_2003}, which does not have a closed-form solution in multivariate settings. This is coupled with a well-known statistical curse of dimensionality for general measures \citep{ajtai1984optimal,dudley1969speed, fournier2015rate, talagrand1992matching, talagrand1994transportation, weed2019sharp}. 

These challenges have impeded the development of a method of projections between potentially high-dimensional probability measures. A focus so far has been on the univariate and low-dimensional setting. In particular, \cite{chen2021wasserstein}, \cite{ghodrati2022distributions}, and \cite{pegoraro2021fast} introduced frameworks for distribution-on-distribution regressions in the univariate setting for object data. \cite{bigot_geodesic_2014, cazelles2017log} developed principal component analyses on the space of univariate probability measures using geodesics on the Wasserstein space.

The most closely related works to ours are \cite{bonneel_wasserstein_2016} and \cite{werenski2022measure}. The former develops a regression approach in barycentric coordinates with applications in computer graphics as well as color and shape transport problems. Their method requires solving a bilevel optimization problem, which is computationally costly and does not need to achieve a global solution. The latter works on a tangential structure like we do, but is based on ``Karcher means'' \citep{karcher2014riemannian, zemel2019frechet}. This implies that their method works between absolutely continuous measures with densities that are bounded away from zero, with the target measure lying in the convex hull of the control measures. 

The notion of projection we propose in this article circumvents these challenges in the multivariate setting by lifting the projection problem to the regular tangent space to $\wasserstein_2$ at the target measure based on generalized geodesics. These tangent spaces exist when the target measure is regular, i.e., if it does not give mass to sets of lower Hausdorff dimension. For general measures, we construct a regular tangent space using barycentric projections \citep[appendix 12]{ambrosio_gradient_2008}.  In contrast to the existing approaches, our method works for general probability measures, allows for the target measure to be outside the generalized geodesic convex hull of the control measures, and can be implemented by a constrained linear regression, which minimizes computational costs. In particular, Propositions \ref{Prop:Solution_Characterization} and \ref{Prop:Solution_Characterization_general} show that our method is a projection of the target onto the generalized geodesic convex hull of the control measures. 

The method hence transforms the projection problem on the positively curved Wasserstein space into a linear optimization problem in the regular tangent space, which provides a unique solution to the projection problem under mild assumptions. This problem takes the form of a deformable template \citep{boissard2015distribution, yuille_deformable_1991}, which connects our approach to this literature. The method can be implemented in three steps: (i) obtain the general tangent cone structure at the target measure, (ii) construct a regular tangent space if it does not exist, and (iii) perform a linear regression to carry out the projection in the tangent space. 

The challenging part of the implementation is lifting the problem to the tangential structure: this requires computing the corresponding optimal transport plans between the target and each measure used in the projection. Many methods have been developed for this, see for instance \citet{benamou2000computational, jacobs2020fast, makkuva2020optimal, peyre2019computational, ruthotto2020machine} and references therein. Other alternatives compute approximations of the optimal transport plans via regularized optimal transport problems \citep{peyre2019computational}, such as entropy regularized optimal transport \citep{galichon2010matching, cuturi2013sinkhorn}. The proposed projection approach is compatible with any such method. We provide results for the statistical consistency when estimating the measures via their empirical counterparts in practice. 

To demonstrate the efficiency and utility of the proposed method, we extend the classical synthetic control estimator \citep{abadie2003economic,abadie2010synthetic} to settings with observed individual heterogeneity in multivariate outcomes. This lets us perform the synthetic control method on the joint distribution of several outcomes and also estimate one set of optimal weights over all pre-intervention time periods. This complements the recently introduced method in \citet{gunsilius_distributional_2021}, which is designed for univariate outcomes. 

We apply this synthetic controls estimator to estimate the causal effect of Medicaid expansion on the population in individuals states. Specifically, we exploit the fact that the Affordable Care Act (ACA) allows individual states to decide whether to adopt such expansion. We use Montana as our target state, which adopted Medicaid expansion in 2016. Using the American Community Survey (ACS; \cite{ruggles2019ipums}) data collected between 2010 and 2016, we evaluate such effects by estimating a counterfactual (``synthetic'') Montana, had the state not adopted Medicaid expansion. We conclude that the policy induces nontrivial, positive effects on Medicaid enrollment, earnings, and labor supply, while its effect on employment agrees with some previous estimates, but much less in magnitude compared to the other effects we estimated.

\section{Methodology}

\subsection{The $2$-Wasserstein space $\wasserstein_2(\mathbb{R}^d)$}
For probability measures $P_X, P_Y\in \mathscr{P}(\mathbb{R}^d)$ with supports $\Chi, \YY \subset \R^d$, respectively, the $2$-Wasserstein distance $W_2 (P_X, P_Y)$ is defined as
\begin{align}\label{Eq:wass}
    W_2 (P_X, P_Y) \coloneqq \left( \underset{\gamma \in \Gamma (P_X, P_Y) }{\min} \int_{\Chi \times \YY} \abs{x-y}^2 \dif \gamma (x,y) \right)^{\frac{1}{2}}.
\end{align}
Here, $|\cdot|$ denotes the Euclidean norm on $\mathbb{R}^d$ and 
\[\Gamma(P_X,P_Y)\coloneqq \left\{\gamma\in \mathscr{P}(\mathbb{R}^d\times\mathbb{R}^d): (\pi_1)_\#\gamma = P_X, \thickspace (\pi_2)_\#\gamma = P_Y\right\}\] is the set of all couplings of $P_X$ and $P_Y$. The maps $\pi_1$ and $\pi_2$ are the projections onto the first and second coordinate, respectively, and $T_\# P$ denotes the pushforward measure of $P$ via $T$, i.e.~for any measurable $A\subset\mathcal{Y}$, $T_\# P(A) \equiv P(T^{-1}(A))$. An optimal coupling $\gamma\in \Gamma(P_X,P_Y)$ solving the optimal transport problem \eqref{Eq:wass} is an \textit{optimal transport plan}. By Prokhorov's theorem, a solution always exists in our setting. 

When $P_X$ is \emph{regular}, i.e.~when it does not give mass to sets of lower Hausdorff dimension in its support, then the optimal transport plan $\gamma$ solving \eqref{Eq:wass} is unique and takes the form $\gamma = (\id \times \nabla\varphi)_\# P_X$, where $\id$ is the identity map on $\mathbb{R}^d$ and $\nabla\varphi(x)$ is the gradient of some convex function. This result is known as Brenier's theorem \citep[Theorem 2.12]{brenier1991polar, mccann1997convexity, villani_topics_2003}. By definition, all measures that possess a density with respect to Lebesgue measure are regular. In the following we will distinguish between regular and general measures. 


The $2$-Wasserstein space $\wasserstein_2\equiv\wasserstein_2(\mathbb{R}^d)\equiv (\PP_2 (\R^d), W_2)$ is the metric space defined on the set $\mathscr{P}_2(\mathbb{R}^d)$ of all probability measures with finite second moments supported on $\mathbb{R}^d$, with the $2$-Wasserstein distance as the metric. It is a complete and separable metric space \citep[Proposition 7.1.5]{ambrosio_gradient_2008} and also possesses a geometric structure that we exploit. In particular, it is a geodesically complete space in the sense that between any two measures $P,P'\in\wasserstein_2$, one can define a geodesic $P_t: [0,1] \to\wasserstein_2$ via the interpolation \citep{ambrosio_gradient_2008, mccann1997convexity} $P_t\coloneqq (\pi_t)_\# \gamma$, where $\gamma$ is an optimal transport plan and $\pi_t: \mathbb{R}^d\times\mathbb{R}^d\to\mathbb{R}^d$ is defined through $\pi_t(x,y)\coloneqq (1-t)x + ty$. Using this, it can be shown that $\wasserstein_2$ is a positively curved metric space $d>1$ \citep[Theorem 7.3.2]{ambrosio_gradient_2008} and flat for $d=1$ \citep{kloeckner2010geometric}, where curvature is defined in the sense of Aleksandrov \citep{aleksandrov1951theorem}. This difference in the curvature properties is the main reason for why the multivariate setting requires different approaches compared to the established results for measures on the real line.

\subsection{Projections in $\wasserstein_2$, barycenters, and existing challenges}
In linear spaces, the idea of projecting an element on a set of other elements is a linear or convex combination of these elements in the set. The analogues of averages in vector spaces are Fr\'echet means or barycenters in metric spaces. 

For Wasserstein spaces, this concept has been introduced in our setting in  \citet{agueh_barycenters_2011} and in a more abstract sense in \citet{carlier2010matching}. For any collection of probability measures $\set{P_j}_{1 \leq j \leq J}\subset\wasserstein_2$, their weighted barycenter $\bar{P}(\lambda)$ for given weights $\lambda\equiv (\lambda_1,\ldots,\lambda_J)\in\Delta^J$ is a solution of the following minimization problem:
\begin{equation}
    \label{Eq:Barycenter}
  \bar{P}(\lambda)\in  \underset{P \in \PP_2 (\R^d)}{\argmin} \sum_{j=1}^{J} \frac{\lambda_j}{2} W^2_2 (P, P_j).
\end{equation}
The weights $\lambda$ are defined to lie in the $J$-dimensional probability simplex $\Delta^{J}$ of nonnegative vectors in $\mathbb{R}^J$ that sum to unity. Prokhorov's theorem implies that a solution to \eqref{Eq:Barycenter} exists \citep[Proposition 2.3]{agueh_barycenters_2011}.

Throughout, we focus on the case where $\lambda\in \Delta^{J}$, because this provides a natural notion of \emph{interpolation} between the given measures $P_j$. We could also construct a linear projection that corresponds to a ``geodesic \emph{extrapolation}'' by relaxing the requirement that all weights $\lambda_j$ need to be non-negative. In the classical setting of random variables mapping to some Euclidean space and not random measures, this relaxation is the natural analogue to a linear regression, see \citet{abadie2015comparative} for a discussion. We focus on the interpolation setting because it is in line with the notion of a projection onto the convex hull spanned by other elements. All of our results can be extended to the extrapolation setting in principle.

One way to extend the notion of projection between a target probability measure $P_0$ and a set of control measures $\{P_j\}_{j=1,\ldots,J}$ in $\wasserstein_2$ would hence consist in finding the optimal weights $\lambda^*\in\Delta^{J}$ for which an induced barycenter $\bar{P}(\lambda^*)$ solving \eqref{Eq:Barycenter} is as close as possible to $P_0$. 
This would lead to the following bi-level optimization problem, assuming that the barycenter $\bar{P}(\lambda)$ is unique for given $\lambda$: 
\begin{equation}
    \label{Eq:BarycenterProblem}
    \lambda^* \in \underset{\lambda \in \Delta^{J}}{\argmin} \mbox{}\thickspace\medspace W_2 (P_0, \bar{P}(\lambda)), \qquad \text{where}\quad \bar{P}(\lambda) = \underset{P \in \PP_2 (\R^d)}{\argmin} \sum_{j=1}^{J} \frac{\lambda_j}{2} W^2_2 (P, P_j).
\end{equation}
A version of this approach is used in \cite{bonneel_wasserstein_2016} to define a notion of regression between probability measures on rectangular supports in low dimensions. The challenges with this approach are mathematical and computational. Importantly, the optimal weights $\lambda^*$ need not be unique. This is not an issue for the applications considered in \cite{bonneel_wasserstein_2016}, like color transport; however, it is important in statistical settings when the weights convey information used in further procedures, like causal inference via synthetic controls, where the optimal weights are used to introduced a counterfactual outcome of a treated unit had it not been treated \citep{abadie2003economic, abadie2010synthetic, abadie2021using}. Moreover, the bi-level optimization structure makes solving the problem prohibitively costly for higher-dimensional distributions. \cite{bonneel_wasserstein_2016} introduce a gradient descent approach based on an entropy-regularized analogue of $W_2$ \citep{cuturi2013sinkhorn, peyre2019computational} that can be implemented in settings with low-dimensional empirical measures of rectangular support. However, in higher dimensions and with general measures, an efficient implementation of \eqref{Eq:BarycenterProblem} that can produce unique weights is absent. 

\subsection{Tangential structures on $\wasserstein_2$}
We circumvent these difficulties by exploiting a tangential structure that can be defined on $\wasserstein_2$ \citep{ambrosio_gradient_2008,otto2001geometry}. In particular, it allows us to entirely circumvent solving a bi-level optimization problem as the one in \eqref{Eq:BarycenterProblem}. 

The tangential structure relies on the fact that geodesics $P_t$ in $\wasserstein_2$ are linear in the transport plans $(\pi_t)_\#\gamma$. This implies a geometric tangent cone structure at each measure $P\in\wasserstein$ that can be defined as the closure in $\PP_2(\mathbb{R}^d)$ of the set 
\[\mathcal{G}(P)\coloneqq \left\{\gamma\in \PP_2(\mathbb{R}^d\times\mathbb{R}^d): (\pi_1)_\#\gamma = P,\thickspace\medspace (\pi_1,\pi_1+\varepsilon\pi_2)_\#\gamma\thickspace\text{is optimal for some $\varepsilon>0$}\right\}\] 
with respect to the local distance 
\begin{equation}\label{Eq:W_P_dist}
W_P^2(\gamma_{12},\gamma_{13})\coloneqq \min\left\{\int_{(\mathbb{R}^d)^3}\left\lvert x_2-x_3\right\rvert^2\dif \gamma_{123}: \gamma_{123}\in\Gamma_1(\gamma_{12},\gamma_{13})\right\},
\end{equation}
where $\gamma_{12}$ and $\gamma_{13}$ are couplings between $P$ and some other measures $P_2$ and $P_3$, respectively, and $\Gamma_1(\gamma_{12},\gamma_{13})$ is the set of all $3$-couplings $\gamma_{123}$ such that the projection of $\gamma_{123}$ onto the first two elements is $\gamma_{12}$ and the projection onto the first and third element is $\gamma_{13}$ \citep[Appendix 12]{ambrosio_gradient_2008}. We can then define the exponential map at $P$ with respect to some tangent element $\gamma\in\mathcal{G}(P)$ by
\begin{equation*}
    \exp_P(\gamma) = (\pi_1+\pi_2)_\#\gamma.
\end{equation*}
This tangent cone can be constructed at every $P\in\wasserstein$, irrespective of its support. 

In the case where $P$ is absolutely continuous with respect to Lebesgue measure, the definition simplifies. The corresponding optimal transport plan between $P$ and any other measure measure $Q\in\wasserstein_2$ is then supported on the graph of the gradient of a convex function $\nabla\varphi$ by Brenier's theorem. This allows the introduction of a regular tangent cone at $P$, defined via \citep[section 8.4]{ambrosio_gradient_2008}
\begin{equation}
    \label{Eq:TangentSpace_Target}
    \tangent_{P} \wasserstein_2 (\mathbb{R}^d) \coloneqq \overline{\left\{t ( \grad \phi_j - \id): \left(\id\times \grad \phi_j\right)_\# P \thickspace\medspace\text{is optimal in $\Gamma(P,(\nabla\varphi_j)_\#P)$},\thickspace\medspace t > 0\right\}}^{L^2(P)},
\end{equation}
where $\overline{A}^{L^2(P)}$ defines the closure of the set $A$ with respect to the distance induced by the $L^2$-norm on $P$.
Interestingly, this tangent cone is actually a linear tangent space \citep[Theorem 8.5.1]{ambrosio_gradient_2008}. In the following, we therefore call $\mathcal{T}_P\wasserstein_2$ the regular tangent space \eqref{Eq:TangentSpace_Target}.

\subsection{Tangential Wasserstein projections for regular target measures}
We want to define a natural analogue to the barycenter projection problem \eqref{Eq:BarycenterProblem} in the regular tangent space of the Wasserstein space. A starting point for this is to consider a characterization of the barycenter $\bar{P}(\lambda)$ for fixed weights of a set $\{P_j\}_{j\in\llbracket J\rrbracket}$ in regular tangent spaces. \citet[Equation (3.10)]{agueh_barycenters_2011} show that if at least one of the measures is absolutely continuous with respect to Lebesgue measure, then $\bar{P}(\lambda)$ can be characterized via 
\begin{equation}
    \label{Eq:BarycenterRequirement}
    \sum_{j=1}^{J} \lambda_j \left(\grad \phi_j - \id\right) = 0,
\end{equation}
where $\{\phi_j\}_{j \in \llbracket J \rrbracket}$ are the optimal transport maps from the barycenter to the respective measure $P_j$, i.e.~$(\phi_j)_\# \bar{P}(\lambda) = P_j$. Each term of the summand in \eqref{Eq:BarycenterRequirement} is an element in $\tangent_{\bar{P}(\lambda)} \wasserstein_2 (\mathbb{R}^d)$ by construction. 

More generally, the condition \eqref{Eq:BarycenterRequirement} is a sufficient condition for $\bar{P}(\lambda)$ to be a ``Karcher mean'' \citep{karcher2014riemannian} in $\wasserstein_2$ \citep{zemel2019frechet}. In fact, a ``Karcher mean'' of a set of measures $\{P_j\}_{j\in\llbracket J\rrbracket}$ is defined as the gradient of the Fr\'echet functional in $\wasserstein_2$ and is characterized through \eqref{Eq:BarycenterRequirement} holding $\bar{P}(\lambda)$-almost everywhere. 
\eqref{Eq:BarycenterRequirement} is a stronger condition because it is assumed to hold at every point in the support of $\bar{P}(\lambda)$, not just almost every point. \citet{alvarez2016fixed} use this characterization to introduce a fixed-point approach to compute Wasserstein barycenters, and \cite{werenski2022measure} use this structure to introduce a replication approach for absolutely continuous measures whose densities are bounded away from zero and whose target measure lies inside the convex hull of the control measures. Related is the recent definition of weak barycenters in \citet{cazelles2021novel}, where the authors replace the optimal transport maps from the classical optimal transport problem by the weak optimal transport problem introduced in \citet{gozlan2017kantorovich}. Heuristically, this characterization is that of a \textit{deformable template}. A measure $P$ is a deformable template if there exists a set of deformations $\set{\psi_{j}}_{j=1,\dots,J}$ such that ${\psi_{j}}_{\#} P = P_j$, in a way that their weighted average is ``as close to the identity'' as possible. In our setting $\psi_j\equiv\nabla\varphi_j-\id$ \citep{anderes_discrete_2015,boissard2015distribution,yuille_deformable_1991}.

In our setting of interest, we are given a target measure $P_0$ that we want to replicate given a set $\{P_j\}_{j\in\llbracket J\rrbracket}$. The key idea for this is to adapt the characterization \eqref{Eq:BarycenterRequirement} \emph{with respect to the target measure} $P_0$. That is, we aim to find the optimal weights $\lambda^*\in\Delta^{J}$ that satisfy
\begin{equation}
    \label{Eq:Barycenter_Solving_Simple_Method}
    \lambda^*\coloneqq \underset{\lambda \in \Delta^{J}}{\argmin} \norm{ \sum_{j=1}^{J} \lambda_j \left( \grad \phi_j - \id \right) }_{L^2(P_0)}^2,
\end{equation}
where $\nabla\varphi_j$ are the optimal transport maps between the target $P_0$ and the control measures $P_j$, $j\in\llbracket J\rrbracket$. 

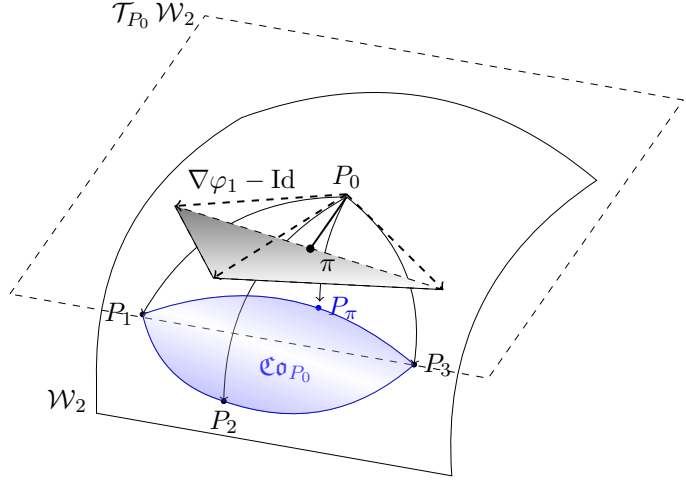
\begin{figure}[H]
\centering
\resizebox{0.55\textwidth}{!}{%
	\begin{tikzpicture}[x={(170:1cm)},y={(55:.7cm)},z={(90:0.8cm)}]
      \draw [dashed, looseness=.4] (3.5,-3.5,0) -- (3.5,3.5,0) -- (-3.5,3.5,0) -- (-3.5,-3.5,0) -- cycle;
      \filldraw[black] (3.5,3.5,0) circle (0pt) node[anchor=east] {{$\mathcal{T}_{P_0} \wasserstein_2$}}; 
      \draw (2.6,-2.6,-2) node[below left] {$\wasserstein_2$}
      to[bend left] (2.6,2.6,-1)
      to[bend left] coordinate (mp) (-2.6,2.6,-1)
      to[bend right] (-2.6,-2.6,-2.6)
      to coordinate (mm) (2.6,-2.6,-2.6)
      -- cycle;
      \draw (0,0,0) node[above] {$P_0$};

      \draw[->] (0,0,0) to [bend right] (2.7,-0.7,-2.2);
      \filldraw[black] (2.7,-0.7,-2.2) circle (1pt) node[anchor=east]{$P_1$};
      \draw[->] (0,0,0) to [bend right] (0.9,-2.2,-2.3);
      \filldraw[black] (0.9,-2.2,-2.3) circle (1pt) node[anchor=north]{$P_2$};
      \draw[->] (0,0,0) to [bend left] (-1.6,-1.5,-1.6);
      \filldraw[black] (-1.6,-1.5,-1.6) circle (1pt) node[anchor=west]{$P_3$};
      %
       \draw[->] (0,0,0) to [bend right=13] (-0.3,-1.7,-0.6);
       \filldraw[blue] (-0.2,-1.5,-0.88) circle (1pt) node[anchor=west]{$P_\pi$};
      \draw node[anchor=north] at (1.7,0.4,0) {$\nabla\varphi_1-\id$};
      \shadedraw  (-1.7,-0.7,-0.8) -- (1.5,-1.1,-1) -- (1.9,-1.5,0.5);
      \draw [dashed, looseness=.4] (-1.7,-0.7,-0.8) -- (1.5,-1.1,-1) -- (1.9,-1.5,0.5) -- cycle;
      \draw [dashed,thick,->] (0,0,0.05) node[below left] {} to (1.9,-1.5,0.5);
      \draw [dashed,thick,->] (0,0,0.05) node[below left] {} to (1.5,-1.1,-1);
      \draw [dashed, thick,->] (0,0,0.05) node[below left] {} to (-1.7,-0.7,-0.8);
      \draw[-, thick] (0,0,0) to (0,-1.3,0);
      \filldraw[black] (0,-1.3,0) circle (1.5pt);
      \draw node[anchor=west] at (0,-1.3,-0.3) {$\pi$};
      \draw[blue] (2.7,-0.7,-2.2) to[bend right] (0.9,-2.2,-2.3);
       \draw[blue] (0.9,-2.2,-2.3) to[bend right] (-1.6,-1.5,-1.6);
       \draw[blue] (-1.6,-1.5,-1.6) to[bend right] (2.7,-0.7,-2.2);
       \draw node[blue,anchor = west, rotate = 0] at (0.4,-2.5,-1.4) {$\convexhull_{P_0}$};
       \shadedraw [left color = blue, right color = blue, middle color = white, opacity = 0.3, shading angle = 340] (2.7,-0.7,-2.2) to[bend right] (0.9,-2.2,-2.3) to [bend right] (-1.6,-1.5,-1.6) to [bend right] (2.7,-0.7,-2.2);
    \end{tikzpicture}
    }
\caption{Tangential Wasserstein projection for a regular target $P_0$.}\label{Fig:wasserstein_projection}
\flushleft
\small
Thick dashed lines are tangent vectors $\nabla\varphi_j-\id$. The gray shaded region is their convex hull in $\mathcal{T}_{P_0}\wasserstein_2$ and $\pi$ is the projection of $\id$ onto this convex hull. $P_\pi\coloneqq \exp_{P_0}(\pi)$ is the projection of $P_0$ onto the generalized geodesic convex hull $\convexhull_{P_0}\left(\{P_1,P_2,P_3\}\right)\subset\wasserstein_2$ (blue). 
\end{figure}

We now show that this approach is in fact a projection of $P_0$ onto the \emph{generalized geodesic convex hull} of the control measures $P_j$ with respect to $P_0$ as illustrated in Figure \ref{Fig:wasserstein_projection}. To define this notion of convex hull, we extend the definition of generalized geodesics \citep[section 9.2]{ambrosio_gradient_2008}, and in particular the definition of $W_P$ to the multimarginal setting, by defining, for given couplings $\gamma_{0j}\in \Gamma(P_0,P_j)$, $j\in\llbracket J\rrbracket$
\begin{equation}\label{W_P-general}
W_{P_0;\lambda}^2(\gamma_{01},\gamma_{02},\ldots,\gamma_{0J})\coloneqq \min\left\{\int_{(\mathbb{R}^d)^{J+1}} \sum_{j=1}^J\lambda_j \left\lvert x_j-x_0\right\rvert^2\dif \bm{\gamma}: \bm{\gamma}\in\Gamma_1(\gamma_{01},\ldots,\gamma_{0J})\right\},
\end{equation}
where $\Gamma_1(\gamma_{01},\ldots,\gamma_{0J})\subset\Gamma(P_0,P_1,\ldots,P_J)$ is the set of all $(J+1)$-couplings $\bm{\gamma}$ such that the projection of $\bm{\gamma}$ onto the first- and $j$-th element is $\gamma_{0j}$.
Note that this definition is similar to the multimarginal definition of the $2$-Wasserstein barycenter \citep{agueh_barycenters_2011, gangbo1998optimal}, but ``centered'' at $P_0$. Based on this, we define the generalized geodesic convex hull of measures $\{P_j\}_{j\in\llbracket J\rrbracket}$ with respect to the measure $P_0$ as

\begin{multline}
    \label{Eq:Gen_cvxhull}
    \convexhull_{P_0} \left( \set{P_j}_{j=1}^{J} \right) \coloneqq\left\{P(\lambda)\in\mathscr{P}_2(\mathbb{R}^d): P(\lambda) = \left(\sum_{j=1}^J\lambda_j\pi_{j+1}\right)_{\raisebox{12pt}{$\scaleobj{0.75}{\#}$}}\bm{\gamma},\right.\\\left. \thickspace\medspace\bm{\gamma}\thickspace\medspace\text{solves $W^2_{P_0;\lambda}(\gamma_{01},\ldots,\gamma_{0J})$}, 
                \vphantom{\left(\sum_{j=1}^J\lambda_j\pi_j\right)_\#\bm{\gamma}} 
        \thickspace\medspace\gamma_{0j}\thickspace\medspace\text{is optimal in $\Gamma(P_0,P_j)$}\thickspace\medspace\forall j\in\llbracket J\rrbracket, \quad \lambda\in \Delta^{J}\right\}.
\end{multline}

Based on these definitions we can show that our approach is a projection of the target $P_0$ onto $\convexhull_{P_0}\left(\{P_j\}_{j=1}^J\right)$.
\begin{proposition}
    \label{Prop:Solution_Characterization}
    Consider a regular target measure $P_0$ and a set $\{P_j\}_{j\in\llbracket J\rrbracket}$ of general control measures. Construct the measure $P_\pi$ as 
    \[P_\pi\coloneqq \exp_{P_0}\left(\sum_{j=1}^J \lambda_j^*(\nabla\varphi_j-\id)\right) ~, \] where the optimal weights $\lambda^*\in \Delta^{J}$ are obtained by solving \eqref{Eq:Barycenter_Solving_Simple_Method} and $\nabla\varphi_j$ are the optimal maps transporting $P_0$ to $P_j$, respectively. Then $P_\pi$ is the unique metric projection of $P_0$ onto $\convexhull_{P_0} \left( \set{P_j}_{j=1}^{J} \right)$. 
\end{proposition}


\subsection{Tangential Wasserstein projections for general target measures}
Proposition \ref{Prop:Solution_Characterization} holds for a regular target $P_0$. In many practical settings, however, the target outcome is not a regular measure, as in our application in Section \ref{sec:medicaid}. In such settings, the corresponding optimal transport problems \eqref{Eq:wass} between the target and the respective control measures $P_j$ is only achieved via optimal transport plans $\gamma_{0j}$, not maps $\nabla\varphi_j$.
In contrast to the regular setting, these transport maps also do not need to be unique. 

Still, the fundamental idea of the tangential projection can be extended to the more general setting, using transport plans instead of maps. The implementation for regular targets \eqref{Eq:Barycenter_Solving_Simple_Method} is a special case of \eqref{W_P-general} where all plans $\gamma_{0j}$ are achieved via transport maps $\nabla\varphi_j$. This implies that the optimal weights $\lambda^*\in\Delta^{J}$ could in principle be obtained by
\begin{equation}\label{Eq:general_tangential}
    \lambda^* \coloneqq \argmin_{\lambda\in \Delta^{J}} W^2_{P_0;\lambda}(\gamma_{01},\ldots,\gamma_{0J}),
\end{equation}
where $\gamma_{0j}$ are optimal transport plans between the target $P_0$ and the respective control measure $P_j$. By definition, a solution to \eqref{Eq:general_tangential} will provide a projection of the target $P_0$ onto $\convexhull_{P_0}\left(\{P_j\}_{j=1}^J\right)$. However, solving \eqref{Eq:general_tangential} is computationally prohibitive in practice for two reasons. First, it is again a bilevel problem, similar to the direct approach \eqref{Eq:BarycenterProblem}. Second, $W_{P_0;\lambda}^2$ requires computing a joint coupling over $J+1$ marginal distributions, which is computationally infeasible in practice for a reasonably large $J$. 

We therefore rely on barycentric projections to reduce the complex general setting to the regular tangent space and subsequently apply the projection \eqref{Eq:Barycenter_Solving_Simple_Method}. This is computationally inexpensive, as it amounts to computing
\begin{equation}
    \label{Eq:Barycenter_Solving_gen}
    \lambda^*\coloneqq \underset{\lambda \in \Delta^{J}}{\argmin} \norm{ \sum_{j=1}^{J} \lambda_j \left( b_{\gamma_{0j}} - \id \right) }_{L^2(P_0)}^2,
\end{equation}
where 
\[b_{\gamma_{0j}}(x_1)\coloneqq \int_{\mathbb{R}^d} x_2\dif \gamma_{0j,x_1}(x_2)\]
are the barycentric projections of optimal transport plans $\gamma_{0j}$ between $P_0$ and $P_j$. Here, $\gamma_{x_1}$ denotes the disintegration of the optimal transport plan $\gamma$ with respect to $P_0$.  

This approach is a natural extension of the regular setting to general probability measures for two reasons. First, if the optimal transport plans $\gamma_{0j}$ are actually induced by some optimal transport map $\nabla\gamma_j$, then $b_{\gamma_{0j}}$ reduces to this optimal transport map; in this case the general tangent cone $\mathcal{G}(P_0)$ reduces to the regular tangent cone $\mathcal{T}_{P_0}\wasserstein_2$ \citep[Theorem 12.4.4]{ambrosio_gradient_2008}. Second, by the definition of $b_{\gamma}$ and disintegrations in conjunction with Jensen's inequality it holds for all $\lambda\in\Delta^{J}$ that
\begin{equation}\label{Eq:contraction}\norm{ \sum_{j=1}^{J} \lambda_j \left( b_{\gamma_{0j}} - \id \right) }_{L^2(P_0)}^2 \leq W^2_{P_0;\lambda}(\gamma_{01},\ldots,\gamma_{0J}).
\end{equation}
This implies that for general $P_0$ we can also define a convex hull based on barycentric projections, which is of the form
\begin{equation}\label{Eq:Gen_cvxhull_maps}
    \widetilde{\convexhull}_{P_0} \left( \set{P_j}_{j=1}^{J} \right) \coloneqq\left\{P(\lambda)\in\mathscr{P}_2(\mathbb{R}^d): P(\lambda) = \left(\sum_{j=1}^J\lambda_j b_{\gamma_{0j}}\right)_{\raisebox{12pt}{$\scaleobj{0.75}{\#}$}}{P_0}, \quad \lambda\in\Delta^{J}\right\}.
\end{equation}
Furthermore, the contraction property \eqref{Eq:contraction} implies that $\convexhull_{P_0}\subset\widetilde{\convexhull}_{P_0}$, with equality when all transport plans are achieved via maps $\nabla\varphi_j$. Using these definitions, the following defines our notion of projection for general $P_0$ and shows that it projects onto $\widetilde{\convexhull}_{P_0}$. 
\begin{proposition}
    \label{Prop:Solution_Characterization_general}
    Consider a general target measure $P_0$ and a set $\{P_j\}_{j\in\llbracket J\rrbracket}$ of general control measures. Construct the measure $\widetilde{P}_\pi$ as 
    \[\widetilde{P}_\pi\coloneqq \exp_{P_0}\left(\sum_{j=1}^J \lambda_j^*b_{\gamma_{0j}}-\id\right) ~, \] where the optimal weights $\lambda^*\in \Delta^{J}$ are obtained by solving \eqref{Eq:Barycenter_Solving_gen} and $\gamma_{0j}$ are optimal plans transporting $P_0$ to $P_j$, respectively. Then for given optimal plans $\gamma_{0j}$, $\widetilde{P}_\pi$ is the unique metric projection of $P_0$ onto $\widetilde{\convexhull}_{P_0} \left( \set{P_j}_{j=1}^{J} \right)$. 
\end{proposition}
%

Note that in contrast to the regular case in Proposition \ref{Prop:Solution_Characterization}, the optimal plans $\gamma_{0j}$ transporting $P_0$ to $P_j$ need not be unique, i.e., the measures $P_j$ might lie outside the cut locus of $P_0$. However, the projection for fixed $\gamma_{0j}$ is unique. 

The proposed method of projections is hence a well-defined notion of a geodesic metric projection: in the case of a regular measure, our approach is a metric projection of the target measure onto the generalized geodesic convex hull made up of the control measures. In the case where the target measure is not regular, we project onto a slight extension of the generalized geodesic convex hull, which we construct by a barycentric projection. The actual projections \eqref{Eq:Barycenter_Solving_Simple_Method} and \eqref{Eq:Barycenter_Solving_gen} are simple regression problems, which are easy to compute in practice once the tangent structure has been constructed. 

\section{Statistical properties of the weights and tangential projection}
We now provide statistical consistency results for our method when the corresponding measures $\{P_j\}_{j\in \llbracket J\rrbracket}$ are estimated from data. We consider the case where the measures $P_j$ are replaced by their empirical counterparts 
\[\prob_{N_j}(A)\coloneqq N_j^{-1}\sum_{n=1}^{N_j} \delta_{X_{n}}(A)\] for every measurable set $A$ in the Borel $\sigma$-algebra on $\mathbb{R}^d$, where $\delta_x(A)$ is the Dirac measure and $\left(X_{1j},\ldots, X_{N_j,j}\right)$ is an independent and identically distributed set of random variables whose distribution is $P_j$. We explicitly allow for different sample sizes $\bigcup_{j=0}^J N_j = N$ for the different measures. To save on notation we write  $\widehat{\varphi}_{N_j}\equiv \widehat{\varphi}_j$, $\widehat{b}_{0j}\equiv \widehat{b}_{\gamma_{0j}, N_j}$ and $\widehat{\gamma}_{0j}\equiv \widehat{\gamma}_{N_j,N_0}$ in the following.

\begin{proposition}[Consistency of the optimal weights]\label{prop:consistency}
Let $\left\{\prob_{N_j}\right\}_{j=0}^J$ be the empirical measures corresponding to the data $\left(X_{1j},\ldots,X_{N_j j}\right)_{j=0}^J$ which are independent and identical draws from $P_j$, respectively, and are supported on some common latent probability space $(\Omega,\mathscr{A},P)$. Assume $P_j$ has finite second moment. As $N_j\to\infty$ for all $j\in\llbracket J\rrbracket$, the corresponding optimal weights $\widehat{\lambda}^*_N = \left(\widehat{\lambda}_{N_1}^*,\ldots, \widehat{\lambda}_{N_J}^*\right)\in \Delta^{J}$ obtained via 
\begin{equation}
    \label{Eq:Barycenter_Solving_gen_emp}
    \widehat{\lambda}^*_N\coloneqq \underset{\lambda \in \Delta^{J}}{\argmin} \norm{ \sum_{j=1}^{J} \lambda_j \left( \widehat{b}_{0j} - \id \right) }_{L^2(\prob_{N_0})}^2,
\end{equation}
satisfy
\[P\left(\left\lvert\widehat{\lambda}^*_N - \lambda^*\right\rvert>\varepsilon\right) \to 0\qquad\text{for all $\varepsilon>0$} ~, \]
where $\lambda^*$ solve \eqref{Eq:Barycenter_Solving_gen}.
\end{proposition} 
This consistency result directly implies consistency of the optimal weights in case the optimal transport problems between $\prob_{N_0}$ and each $\prob_{N_j}$ are achieved by optimal transport maps $\nabla\widehat{\varphi}_{N_j}$. 
Based on this we also have a consistency result for the empirical counterparts $\widetilde{\prob}_{\pi,N}$ of the optimal projection $\widetilde{P}_\pi$. 
\begin{corollary}[Consistency of the optimal projections]\label{cor:consistency}
In the setting of Proposition \ref{prop:consistency}, the estimated projections $\widetilde{\prob}_{\pi,N}$ converge weakly in probability to the projection $\widetilde{P}_\pi$ as $N_j\to\infty$ for all $j\in\llbracket J\rrbracket$. 
\end{corollary}

Proposition \ref{prop:consistency} and Corollary \ref{cor:consistency} hold in all generality and without any assumptions on the corresponding measures $P_j$, except that they possess finite second moments. To get stronger results, for instance parametric rates of convergences or even asymptotic Gaussianity, one needs to make strong regularity assumptions on the measures $P_j$. Without these, the rate of convergence of optimal transport maps in terms of expected square loss is as slow as $n^{-2/d}$ \citep{hutter2021minimax}. Under such additional regularity conditions, the results for the asymptotic properties are standard, because the proposed method reduces to a classical semiparametric estimation problem, as the weights $\lambda_j$ are finite-dimensional. In particular, the setting is that of a MINPIN estimator, as defined in \citet{andrews1994asymptotics}. 

The fact that the actual tangential projection is a simple constrained linear regression for the weights $\lambda$ implies that the regularity of $b_j$ is what drives the statistical properties of the weights $\lambda$. For instance, if the barycentric projections $b_j$ are regular in the sense that they lie in a Donsker class for their given dimension $d$ \citep{wellner2013weak} and converge to their population counterpart at at least the rate of $n^{1/4}$, then the optimal weights will converge at the parametric rate, which follows directly from the arguments in  \cite{andrews1989asymptotics, andrews1994asymptotics}. In the regular setting, i.e.~when the target $P_0$ is a regular measure so that $b_j=\nabla\varphi_j$ for some convex functions $\varphi_j$, such regularity conditions can be derived from classical regularity theory \citep{caffarelli1990interior, caffarelli1992regularity, de2013w}, and have been used in deriving rates of convergences of optimal transport maps by \cite{deb2021rates, forrow2019statistical, gunsilius2021convergence, hutter2021minimax, manole2021plugin, weed2019sharp}. The same holds for estimators that use barycentric projections after solving an entropy-regularized analogue to the optimal transport problem \citep{seguy2018large, pooladian_entropic_2021}.  Without these regularity assumptions, the curse of dimensionality outlined in these statistical results implies that the rate of convergence for the optimal weights will in general be slower than the parametric rate, especially in higher dimensions.


\section{Simulation and Application}
In this section, we provide some simulations and an application to the synthetic controls estimator to demonstrate the computational properties of the method. We use the \texttt{POT} package \citep{flamary2021pot} to obtain the optimal transport plans, and \texttt{CVXPY} \citep{diamond2016cvxpy, agrawal2018rewriting} to compute \eqref{Eq:Barycenter_Solving_Simple_Method}. Additional details of our applications are contained in Appendix \ref{Appendix:additional-details-application}.

\subsection{Simulation: Gaussian Distributions}

We apply our estimator to Gaussian distributions in dimension $d = 10$. We draw from the following Gaussians:
    \[\mathbf{X}_j \sim \Normal \left( \mu_j, \Sigma \right), \quad j=0,1,2, 3 ~, \]
where $\mu_0 = [10, 10, \dots, 10]$, $\mu_1 = [50, 50, \dots, 50]$, $\mu_2 = [200, 200, \dots, 200]$, $\mu_3 = [-50, -50, \dots, -50]$ and $\Sigma = \id_{10}+ 0.8 \id^{-}_{10}$, with $\id_{10}^{-}$ the $10\times 10$ matrix with zeros on the main diagonal and ones on all off-diagonal terms. We take $\mathbf{X}_0$ as target, and $\mathbf{X}_1$, $\mathbf{X}_2$, $\mathbf{X}_3$ as controls. The optimal weights are $\lambda^* =[0.3643, 0.0943, 0.5414]$, meaning $\mathbf X_1$ and $\mathbf X_3$ receive substantial weights, while $\mathbf X_2$ only receives a small amount. This weight distribution can be understood by looking at the differences between $\mu_0$ and $\mu_1$, $\mu_2$ and $\mu_3$, separately---which, in this case, is based on the distance between the means of these distributions, since we work with Gaussian distributions of the same variance. The mean of $\mathbf X_2$ is significantly further away from the mean of the target than the other two means, so it is to be expected that $\mathbf X_2$ receives little weight. Table \ref{tab:mean-gaussian} suggests the projection is close to the target distribution when only considering the mean, despite only having three control units.

\begin{table}[H]
    \centering
    \footnotesize
    \begin{tabular}{|c|c|c|c|c|c|c|c|c|c|c|}
        \hline
        $d$ & 1 & 2 & 3 & 4 & 5 & 6 & 7 & 8 & 9 & 10 \\
        \hline
        mean$(\mathbf X_0)$ & 10.011 & 9.998 & 10.018 & 10.000 & 10.016 & 9.990 & 10.007 & 10.002 & 10.006 & 10.002 \\
        \hline
        mean$(\sum_{j=1}^{3} \lambda_j^* \mathbf X_j)$ & 10.006 & 10.001 & 10.006 & 10.001 & 10.007 & 10.008 & 10.002 & 10.003 & 10.005 & 10.001 \\
        \hline
    \end{tabular}
    \caption{Means of target and optimally-weighted controls $\mathbf X_1$, $\mathbf X_2$, $\mathbf X_3$.}
    \label{tab:mean-gaussian}
\end{table}


\subsection{Simulation: Replicating Images}

To demonstrate the property that our method provides sparse weights in general, we provide an application on replicating a target image of an object using images of the same object taken from different angles. We use the Lego Bricks dataset available from \texttt{Kaggle}, which contains approximately 12,700 images of 16 different Lego bricks. These images are in the RGBA format, despite being grayscale, which provides a good resolution. For our application, the target image is contained in Figure \ref{fig:lego}\textsc{(A)}, and the control images are in Figure \ref{fig:lego_controls}. 

The result, as shown in Figure \ref{fig:lego}, indicates the optimally-weighted projection replicates the target image well, even in a setting with only a few controls. We note two interesting findings. One, only the control images in the first row of Figure \ref{fig:lego_controls} received nontrivial weights; all others received essentially zero weights. This suggests that our method use most information from controls which look sufficiently like the target and demonstrates that our method provides sparse weights. Two, even with few controls, the optimally weighted projection approximates the target well in this application.

\begin{figure}[H]
    \centering
    \begin{subfigure}{0.25\textwidth}
         \centering
         \includegraphics[width=\textwidth]{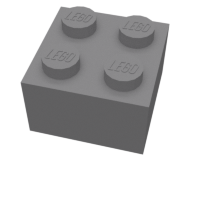}
         \caption{Target block}
         \label{fig:target}
    \end{subfigure}
    \begin{subfigure}{0.25\textwidth}
         \centering
         \includegraphics[width=\textwidth]{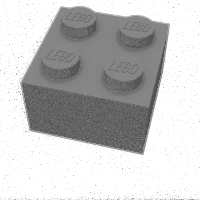}
         \caption{Projection}
         \label{fig:empstat}
    \end{subfigure}
    \caption{Target Lego block and projection.}
    \label{fig:lego}
\end{figure}

\begin{figure}[H]
\captionsetup[subfigure]{labelformat=empty}
\captionsetup[subfigure]{justification=centering}
    \centering
    \begin{subfigure}{0.17\textwidth}
         \centering
         \includegraphics[width=\textwidth]{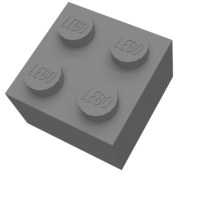}
         \caption{0.059}
    \end{subfigure}
    \begin{subfigure}{0.17\textwidth}
         \centering
         \includegraphics[width=\textwidth]{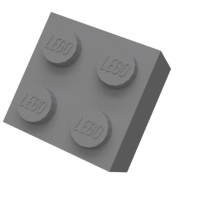}
         \caption{0.313}
    \end{subfigure}
    \begin{subfigure}{0.17\textwidth}
         \centering
         \includegraphics[width=\textwidth]{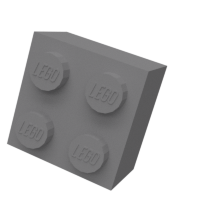}
         \caption{0.077}
    \end{subfigure}
    \begin{subfigure}{0.17\textwidth}
         \centering
         \includegraphics[width=\textwidth]{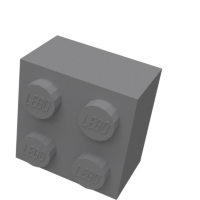}
         \caption{0.083}
    \end{subfigure}
    \begin{subfigure}{0.17\textwidth}
         \centering
         \includegraphics[width=\textwidth]{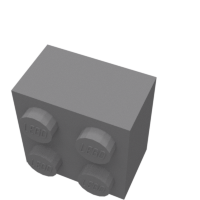}
         \caption{0.468}
    \end{subfigure}

    \begin{subfigure}{0.17\textwidth}
         \centering
         \includegraphics[width=\textwidth]{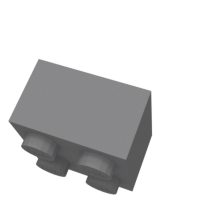}
         \caption{0}
    \end{subfigure}\begin{subfigure}{0.17\textwidth}
         \centering
         \includegraphics[width=\textwidth]{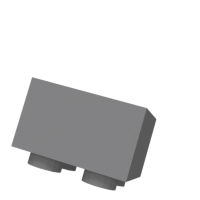}
         \caption{0}
    \end{subfigure}\begin{subfigure}{0.17\textwidth}
         \centering
         \includegraphics[width=\textwidth]{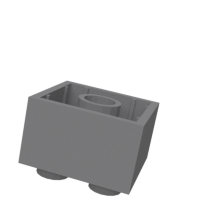}
         \caption{0}
    \end{subfigure}\begin{subfigure}{0.17\textwidth}
         \centering
         \includegraphics[width=\textwidth]{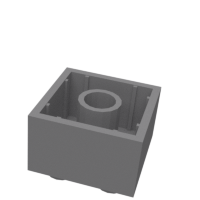}
         \caption{0}
    \end{subfigure}\begin{subfigure}{0.17\textwidth}
         \centering
         \includegraphics[width=\textwidth]{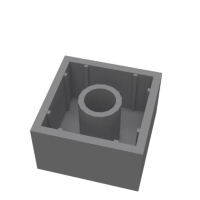}
         \caption{0}
    \end{subfigure}
    \caption{Control units used in simulation and their respective weights.}
    \label{fig:lego_controls}
\end{figure}

\subsection{Application: Medicaid expansion}
%
%
\label{sec:medicaid}
In this section, we apply the method to extend the classical notion of synthetic controls \citep{abadie2003economic, abadie2010synthetic, abadie2021using} and its generalization to univariate distributions \citep{gunsilius_distributional_2021} to general multivariate outcome distributions. Moreover, this generalization allows to estimate one set of optimal weights over all time periods jointly, while existing methods need to apply the method in every time period before averaging. This removes sparsity in the optimal weights estimated. For more details, we refer to \citet{abadie2021using} and \citet{gunsilius_distributional_2021}.

An application to a causal inference setting is studying the effect of health insurance coverage following state-level Medicaid expansion in the United States. A provision within the ACA allows states to decide whether to expand Medicaid for low-income households. Some states decide to adopt such expansions early on, while others did not (and still have not done so). We investigate the economic and behavioral effects of Medicaid expansion. Specifically, we consider some first-order effects (i.e. the extensive margin of Medicaid enrollment post-expansion and disemployment effects) and second-order effects (i.e. income and labor supply effects) of expanded Medicaid access.


The observational data we use is the ACS. From it, we collect the following variables:
\begin{enumerate}
    \item HINSCAID, which indicates whether the individual interviewed is covered by Medicaid, 
    
    \item EMPSTAT, which indicates employment status (either employed or unemployed),
    
    \item UHRSWORK, which indicates number of hours worked,
    
    \item INCWAGE, which indicates previous-twelve-month labor income level.
\end{enumerate}
We measure labor hours supplied and labor income in logs instead of levels. We consider Montana as the treated unit for this application. Montana adopted such expansion in 2016. For control units, we use the twelve states for which such expansion has never occurred. As of 2022, the twelve states are: Alabama, Florida, Georgia, Kansas, Mississippi, North Carolina, South Carolina, South Dakota, Tennessee, Texas, Wisconsin, Wyoming. We use data from 2010 to 2016 to estimate optimal weights for the control states in order to create the ``synthetic Montana'', i.e. Montana had it not adopted Medicaid expansion. Details of sample selection and estimating ``synthetic Montana'' are described in Appendix \ref{Appendix:Details_Medicaid}. As we show in the Appendix, the combination of control states replicating Montana is close to actual Montana in the pre-intervention period, implying that the observed differences in post-intervention periods can be attributed to the causal effect of the Medicaid expansion.

\begin{figure}[H]
    \centering
    \begin{subfigure}{0.47\textwidth}
         \centering
         \includegraphics[width=\textwidth]{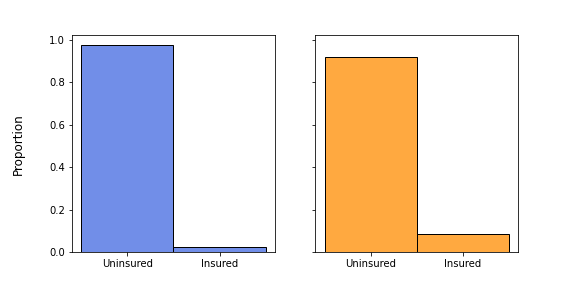}
         \caption{Covered by Medicaid}
         \label{fig:hinscaid}
    \end{subfigure}
    \hfill
    \begin{subfigure}{0.47\textwidth}
         \centering
         \includegraphics[width=\textwidth]{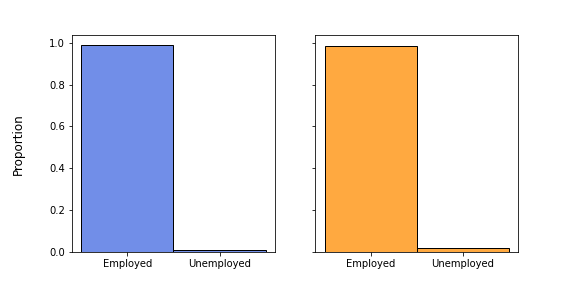}
         \caption{Employment Status}
         \label{fig:empstat-2}
    \end{subfigure}
    \caption{Counterfactual (blue) vs actual (orange) Montana from 2017 to 2020.}
    \label{fig:first-order-results}
\end{figure}

\begin{figure}[H]
    \centering
    \begin{subfigure}{0.47\textwidth}
         \centering
         \includegraphics[width=\textwidth]{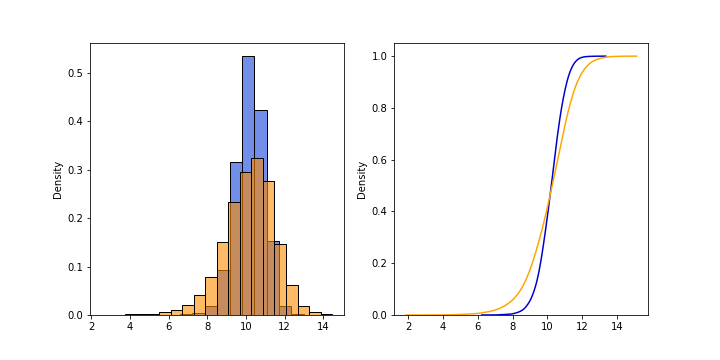}
         \caption{Log Wage}
         \label{fig:hinscaid-2}
    \end{subfigure}
    \hfill
    \begin{subfigure}{0.47\textwidth}
         \centering
         \includegraphics[width=\textwidth]{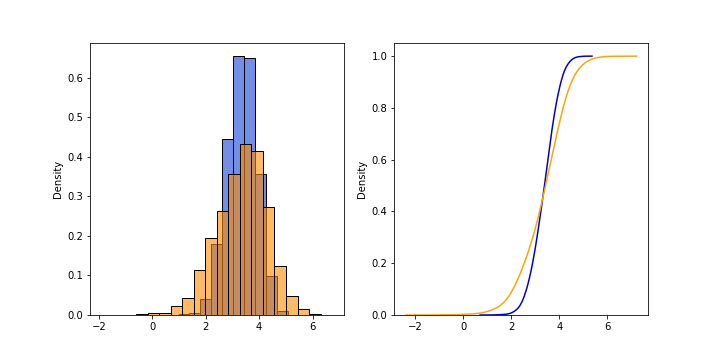}
         \caption{Log Labor Hours Supplied}
         \label{fig:empstat-3}
    \end{subfigure}
    \caption{Counterfactual (blue) vs actual (orange) Montana from 2017 to 2020. In each panel, histograms of data distributions are shown on the left, and cumulative distribution functions are shown on the right.}
    \label{fig:second-order-results}
\end{figure}

\begin{table}[H]
    \centering
    \footnotesize
    \begin{tabular}{|c|c|c|c|c|c|c|c|c|c|c|c|c|}
        \hline
        State & AL & FL & GA & KS & MS & NC & SC & SD & TN & TX & WI & WY \\
        \hline
        Weight & 0.184 & 0 & 0 & 0 & 0.174 & 0 & 0.010 & 0.513 & 0 & 0 & 0.119 & 0\\
        \hline
    \end{tabular}
    \caption{Estimated Weights for Control States.}
    \label{tab:synthetic-weights}
\end{table}

Consistent with findings in \cite{courtemanche2017early, mazurenko2018effects}, we find significant first-order effects of Medicaid expansion, which are summarized in Figure \ref{fig:first-order-results}. We note that the ``synthetic Montana'' has much lower proportion of individuals insured under Medicaid, suggesting that expanding Medicaid eligibility directly affects the extensive margin of Medicaid enrollment. The disemployment effect is much less pronounced in comparison to the enrollment effect we estimated, but nonetheless positive and nontrivial; this is consistent with the findings in, e.g., \cite{peng2020effects}, but inconsistent with those in, e.g., \cite{gooptu2016medicaid}.

We also find nontrivial, positive second-order effects, summarized in Figure \ref{fig:second-order-results}. In both earnings and labor hours supplied, we see the ``synthetic Montana'' has lower averages, and narrower supports compared to the observed distributions. This suggests Medicaid expansion improves earnings, and widens the intensive margin of labor supply.

We further note two crucial findings from our application. One, the optimal weights estimated here is, again, sparse. Based on our results, summarized in Table \ref{tab:synthetic-weights}, we observed only five control states---Alabama, Mississippi, South Carolina, South Dakota, and Wisconsin---with nonzero weights. South Dakota alone constitutes over half of the total weight, suggesting it best approximates Montana compared to all other control states. Two, we estimated these weights using data from all years covered in the pre-intervention period, which provides one sparse set of weights over all time periods. This stands in contrast to the standard synthetic controls method \citep{abadie2003economic,abadie2010synthetic}, where the optimal weights are obtained from taking some weighted average of weights estimated in every time unit during the pre-intervention period; this averaging over weights in each time period generates a non-sparse weight. 



\section{Conclusion}

We have developed a projection method between sets of probability measures supported on $\mathbb{R}^d$ based on the tangential structure of the 2-Wasserstein space. Our method seeks to best approximate some target distribution that is potentially multivariate, using some chosen set of control distributions. We provide an implementation which gives unique, interpretable weights in a setting of regular probability measures. For general probability measures, we construct our projection by first creating a regular tangent space through applying barycentric projection to optimal transport plans. Our application to evaluating the first- and second-order effects of Medicaid expansion in Montana via an extension of the synthetic controls estimator demonstrates the method's efficiency and the necessity to have a method that is applicable for general proabbility measures. The approach still works without restricting optimal weights to be in the unit simplex, which would allow for extrapolation beyond the convex hull of the control units, providing a notion of tangential regression. It can also be extended to a continuum of measures, using established consistency results of barycenters \citep[e.g.][]{le2017existence}.


\bibliographystyle{apalike}
\bibliography{DSC_bib}

\begin{thebibliography}{}

\bibitem[Abadie, 2021]{abadie2021using}
Abadie, A. (2021).
\newblock Using synthetic controls: Feasibility, data requirements, and
  methodological aspects.
\newblock {\em Journal of Economic Literature}, 59(2):391--425.

\bibitem[Abadie et~al., 2010]{abadie2010synthetic}
Abadie, A., Diamond, A., and Hainmueller, J. (2010).
\newblock Synthetic control methods for comparative case studies: Estimating
  the effect of california’s tobacco control program.
\newblock {\em Journal of the American statistical Association},
  105(490):493--505.

\bibitem[Abadie et~al., 2015]{abadie2015comparative}
Abadie, A., Diamond, A., and Hainmueller, J. (2015).
\newblock Comparative politics and the synthetic control method.
\newblock {\em American Journal of Political Science}, 59(2):495--510.

\bibitem[Abadie and Gardeazabal, 2003]{abadie2003economic}
Abadie, A. and Gardeazabal, J. (2003).
\newblock The economic costs of conflict: A case study of the basque country.
\newblock {\em American economic review}, 93(1):113--132.

\bibitem[Agrawal et~al., 2018]{agrawal2018rewriting}
Agrawal, A., Verschueren, R., Diamond, S., and Boyd, S. (2018).
\newblock A rewriting system for convex optimization problems.
\newblock {\em Journal of Control and Decision}, 5(1):42--60.

\bibitem[Agueh and Carlier, 2011]{agueh_barycenters_2011}
Agueh, M. and Carlier, G. (2011).
\newblock Barycenters in the {Wasserstein} {Space}.
\newblock {\em SIAM Journal on Mathematical Analysis}, 43(2):904--924.

\bibitem[Ajtai et~al., 1984]{ajtai1984optimal}
Ajtai, M., Koml{\'o}s, J., and Tusn{\'a}dy, G. (1984).
\newblock On optimal matchings.
\newblock {\em Combinatorica}, 4(4):259--264.

\bibitem[Aleksandrov, 1951]{aleksandrov1951theorem}
Aleksandrov, A.~D. (1951).
\newblock A theorem on triangles in a metric space and some of its
  applications.
\newblock {\em Trudy Matematicheskogo Instituta imeni VA Steklova}, 38:5--23.

\bibitem[Aliprantis and Border, 1999]{aliprantis_infinite_1999}
Aliprantis, C.~D. and Border, K.~C. (1999).
\newblock {\em Infinite dimensional analysis: a hitchhiker's guide}.
\newblock Springer, Berlin ; New York, 2nd, completely rev. and enl. ed
  edition.

\bibitem[{\'A}lvarez-Esteban et~al., 2016]{alvarez2016fixed}
{\'A}lvarez-Esteban, P.~C., Del~Barrio, E., Cuesta-Albertos, J., and
  Matr{\'a}n, C. (2016).
\newblock A fixed-point approach to barycenters in wasserstein space.
\newblock {\em Journal of Mathematical Analysis and Applications},
  441(2):744--762.

\bibitem[Ambrosio et~al., 2008]{ambrosio_gradient_2008}
Ambrosio, L., Gigli, N., and Savaré, G. (2008).
\newblock {\em Gradient flows in metric spaces and in the space of probability
  measures}.
\newblock Lectures in mathematics {ETH} {Zürich}. Birkhäuser, Basel, 2. ed
  edition.

\bibitem[Anderes et~al., 2015]{anderes_discrete_2015}
Anderes, E., Borgwardt, S., and Miller, J. (2015).
\newblock Discrete {Wasserstein} {Barycenters}: {Optimal} {Transport} for
  {Discrete} {Data}.
\newblock {\em arXiv:1507.07218 [math]}.
\newblock arXiv: 1507.07218.

\bibitem[Andrews, 1989]{andrews1989asymptotics}
Andrews, D. (1989).
\newblock Asymptotics for semiparametric econometric models: Iii. testing and
  examples.
\newblock Technical report, Cowles Foundation for Research in Economics, Yale
  University.

\bibitem[Andrews, 1994]{andrews1994asymptotics}
Andrews, D.~W. (1994).
\newblock Asymptotics for semiparametric econometric models via stochastic
  equicontinuity.
\newblock {\em Econometrica: Journal of the Econometric Society}, pages 43--72.

\bibitem[Athey and Imbens, 2015]{athey2015machine}
Athey, S. and Imbens, G.~W. (2015).
\newblock Machine learning methods for estimating heterogeneous causal effects.
\newblock {\em stat}, 1050(5):1--26.

\bibitem[Benamou and Brenier, 2000]{benamou2000computational}
Benamou, J.-D. and Brenier, Y. (2000).
\newblock A computational fluid mechanics solution to the monge-kantorovich
  mass transfer problem.
\newblock {\em Numerische Mathematik}, 84(3):375--393.

\bibitem[Bigot et~al., 2014]{bigot_geodesic_2014}
Bigot, J., Gouet, R., Klein, T., and López, A. (2014).
\newblock Geodesic {PCA} in the {Wasserstein} space.
\newblock {\em arXiv:1307.7721 [math, stat]}.
\newblock arXiv: 1307.7721.

\bibitem[Boissard et~al., 2015]{boissard2015distribution}
Boissard, E., Le~Gouic, T., and Loubes, J.-M. (2015).
\newblock Distribution’s template estimate with wasserstein metrics.
\newblock {\em Bernoulli}, 21(2):740--759.

\bibitem[Bonneel et~al., 2016]{bonneel_wasserstein_2016}
Bonneel, N., Peyré, G., and Cuturi, M. (2016).
\newblock Wasserstein barycentric coordinates: histogram regression using
  optimal transport.
\newblock {\em ACM Transactions on Graphics}, 35(4):1--10.

\bibitem[Brenier, 1991]{brenier1991polar}
Brenier, Y. (1991).
\newblock Polar factorization and monotone rearrangement of vector-valued
  functions.
\newblock {\em Communications on pure and applied mathematics}, 44(4):375--417.

\bibitem[Caffarelli, 1990]{caffarelli1990interior}
Caffarelli, L.~A. (1990).
\newblock Interior w2, p estimates for solutions of the monge-ampere equation.
\newblock {\em Annals of Mathematics}, pages 135--150.

\bibitem[Caffarelli, 1992]{caffarelli1992regularity}
Caffarelli, L.~A. (1992).
\newblock The regularity of mappings with a convex potential.
\newblock {\em Journal of the American Mathematical Society}, 5(1):99--104.

\bibitem[Carlier and Ekeland, 2010]{carlier2010matching}
Carlier, G. and Ekeland, I. (2010).
\newblock Matching for teams.
\newblock {\em Economic theory}, 42(2):397--418.

\bibitem[Cazelles et~al., 2017]{cazelles2017log}
Cazelles, E., Seguy, V., Bigot, J., Cuturi, M., and Papadakis, N. (2017).
\newblock Log-pca versus geodesic pca of histograms in the wasserstein space.
\newblock {\em arXiv preprint 1708.08143}.

\bibitem[Cazelles et~al., 2021]{cazelles2021novel}
Cazelles, E., Tobar, F., and Fontbona, J. (2021).
\newblock A novel notion of barycenter for probability distributions based on
  optimal weak mass transport.
\newblock {\em Advances in Neural Information Processing Systems}, 34.

\bibitem[Chen et~al., 2021]{chen2021wasserstein}
Chen, Y., Lin, Z., and M{\"u}ller, H.-G. (2021).
\newblock Wasserstein regression.
\newblock {\em Journal of the American Statistical Association}, pages 1--14.

\bibitem[Courtemanche et~al., 2017]{courtemanche2017early}
Courtemanche, C., Marton, J., Ukert, B., Yelowitz, A., and Zapata, D. (2017).
\newblock Early impacts of the affordable care act on health insurance coverage
  in medicaid expansion and non-expansion states.
\newblock {\em Journal of Policy Analysis and Management}, 36(1):178--210.

\bibitem[Cuturi, 2013]{cuturi2013sinkhorn}
Cuturi, M. (2013).
\newblock Sinkhorn distances: Lightspeed computation of optimal transport.
\newblock In {\em Advances in {N}eural {I}nformation {P}rocessing {S}ystems},
  pages 2292--2300.

\bibitem[{de Philippis} and Figalli, 2013]{de2013w}
{de Philippis}, G. and Figalli, A. (2013).
\newblock {$W^{2,1}$} regularity for solutions of the monge--amp{\`e}re
  equation.
\newblock {\em Inventiones mathematicae}, 192(1):55--69.

\bibitem[Deb et~al., 2021]{deb2021rates}
Deb, N., Ghosal, P., and Sen, B. (2021).
\newblock Rates of estimation of optimal transport maps using plug-in
  estimators via barycentric projections.
\newblock {\em Advances in Neural Information Processing Systems}, 34.

\bibitem[Diamond and Boyd, 2016]{diamond2016cvxpy}
Diamond, S. and Boyd, S. (2016).
\newblock {CVXPY}: {A} {P}ython-embedded modeling language for convex
  optimization.
\newblock {\em Journal of Machine Learning Research}, 17(83):1--5.

\bibitem[Dudley, 1969]{dudley1969speed}
Dudley, R.~M. (1969).
\newblock The speed of mean glivenko-cantelli convergence.
\newblock {\em The Annals of Mathematical Statistics}, 40(1):40--50.

\bibitem[Dudley, 2018]{dudley2018real}
Dudley, R.~M. (2018).
\newblock {\em Real analysis and probability}.
\newblock CRC Press.

\bibitem[Flamary et~al., 2021]{flamary2021pot}
Flamary, R., Courty, N., Gramfort, A., Alaya, M.~Z., Boisbunon, A., Chambon,
  S., Chapel, L., Corenflos, A., Fatras, K., Fournier, N., Gautheron, L.,
  Gayraud, N.~T., Janati, H., Rakotomamonjy, A., Redko, I., Rolet, A., Schutz,
  A., Seguy, V., Sutherland, D.~J., Tavenard, R., Tong, A., and Vayer, T.
  (2021).
\newblock Pot: Python optimal transport.
\newblock {\em Journal of Machine Learning Research}, 22(78):1--8.

\bibitem[Forrow et~al., 2019]{forrow2019statistical}
Forrow, A., H{\"u}tter, J.-C., Nitzan, M., Rigollet, P., Schiebinger, G., and
  Weed, J. (2019).
\newblock Statistical optimal transport via factored couplings.
\newblock In {\em The 22nd International Conference on Artificial Intelligence
  and Statistics}, pages 2454--2465. PMLR.

\bibitem[Fournier and Guillin, 2015]{fournier2015rate}
Fournier, N. and Guillin, A. (2015).
\newblock On the rate of convergence in wasserstein distance of the empirical
  measure.
\newblock {\em Probability Theory and Related Fields}, 162(3):707--738.

\bibitem[Galichon and Salani{\'e}, 2010]{galichon2010matching}
Galichon, A. and Salani{\'e}, B. (2010).
\newblock Matching with trade-offs: Revealed preferences over competing
  characteristics.
\newblock CEPR Discussion Paper No. DP7858.

\bibitem[Gangbo and {\'S}wi{\k{e}}ch, 1998]{gangbo1998optimal}
Gangbo, W. and {\'S}wi{\k{e}}ch, A. (1998).
\newblock Optimal maps for the multidimensional monge-kantorovich problem.
\newblock {\em Communications on Pure and Applied Mathematics: A Journal Issued
  by the Courant Institute of Mathematical Sciences}, 51(1):23--45.

\bibitem[Ghodrati and Panaretos, 2022]{ghodrati2022distributions}
Ghodrati, L. and Panaretos, V.~M. (2022).
\newblock {Distribution-on-distribution regression via optimal transport maps}.
\newblock {\em Biometrika}.
\newblock asac005.

\bibitem[Gooptu et~al., 2016]{gooptu2016medicaid}
Gooptu, A., Moriya, A.~S., Simon, K.~I., and Sommers, B.~D. (2016).
\newblock Medicaid expansion did not result in significant employment changes
  or job reductions in 2014.
\newblock {\em Health affairs}, 35(1):111--118.

\bibitem[Gozlan et~al., 2017]{gozlan2017kantorovich}
Gozlan, N., Roberto, C., Samson, P.-M., and Tetali, P. (2017).
\newblock Kantorovich duality for general transport costs and applications.
\newblock {\em Journal of Functional Analysis}, 273(11):3327--3405.

\bibitem[Gunsilius, 2022]{gunsilius_distributional_2021}
Gunsilius, F. (2022).
\newblock Distributional synthetic controls.
\newblock {\em arXiv:2001.06118 [econ, stat]}.
\newblock arXiv: 2001.06118.

\bibitem[Gunsilius, 2021]{gunsilius2021convergence}
Gunsilius, F.~F. (2021).
\newblock On the convergence rate of potentials of {B}renier maps.
\newblock {\em Econometric Theory}.
\newblock To appear.

\bibitem[H{\"u}tter and Rigollet, 2021]{hutter2021minimax}
H{\"u}tter, J.-C. and Rigollet, P. (2021).
\newblock Minimax estimation of smooth optimal transport maps.
\newblock {\em The Annals of Statistics}, 49(2):1166--1194.

\bibitem[Jacobs and L{\'e}ger, 2020]{jacobs2020fast}
Jacobs, M. and L{\'e}ger, F. (2020).
\newblock A fast approach to optimal transport: The back-and-forth method.
\newblock {\em Numerische Mathematik}, 146(3):513--544.

\bibitem[Karcher, 2014]{karcher2014riemannian}
Karcher, H. (2014).
\newblock Riemannian center of mass and so called karcher mean.
\newblock {\em arXiv preprint arXiv:1407.2087}.

\bibitem[Kloeckner, 2010]{kloeckner2010geometric}
Kloeckner, B. (2010).
\newblock A geometric study of wasserstein spaces: Euclidean spaces.
\newblock {\em Annali della Scuola Normale Superiore di Pisa-Classe di
  Scienze}, 9(2):297--323.

\bibitem[Le~Gouic and Loubes, 2017]{le2017existence}
Le~Gouic, T. and Loubes, J.-M. (2017).
\newblock Existence and consistency of wasserstein barycenters.
\newblock {\em Probability Theory and Related Fields}, 168(3):901--917.

\bibitem[Makkuva et~al., 2020]{makkuva2020optimal}
Makkuva, A., Taghvaei, A., Oh, S., and Lee, J. (2020).
\newblock Optimal transport mapping via input convex neural networks.
\newblock In {\em International Conference on Machine Learning}, pages
  6672--6681. PMLR.

\bibitem[Manole et~al., 2021]{manole2021plugin}
Manole, T., Balakrishnan, S., Niles-Weed, J., and Wasserman, L. (2021).
\newblock Plugin estimation of smooth optimal transport maps.
\newblock arXiv preprint 2107.12364.

\bibitem[Marron and Alonso, 2014]{marron2014overview}
Marron, J.~S. and Alonso, A.~M. (2014).
\newblock Overview of object oriented data analysis.
\newblock {\em Biometrical Journal}, 56(5):732--753.

\bibitem[Mazurenko et~al., 2018]{mazurenko2018effects}
Mazurenko, O., Balio, C.~P., Agarwal, R., Carroll, A.~E., and Menachemi, N.
  (2018).
\newblock The effects of medicaid expansion under the aca: a systematic review.
\newblock {\em Health Affairs}, 37(6):944--950.

\bibitem[McCann, 1997]{mccann1997convexity}
McCann, R.~J. (1997).
\newblock A convexity principle for interacting gases.
\newblock {\em Advances in mathematics}, 128(1):153--179.

\bibitem[Newey and McFadden, 1994]{newey1994large}
Newey, W.~K. and McFadden, D. (1994).
\newblock Large sample estimation and hypothesis testing.
\newblock {\em Handbook of econometrics}, 4:2111--2245.

\bibitem[Otto, 2001]{otto2001geometry}
Otto, F. (2001).
\newblock The geometry of dissipative evolution equations: The porous medium
  equation.
\newblock {\em Communications in Partial Differential Equations},
  26(1-2):101--174.

\bibitem[Pegoraro and Beraha, 2021]{pegoraro2021fast}
Pegoraro, M. and Beraha, M. (2021).
\newblock Fast pca in 1-d wasserstein spaces via b-splines representation and
  metric projection.
\newblock In {\em Proceedings of the AAAI Conference on Artificial
  Intelligence}, volume~35, pages 9342--9349.

\bibitem[Peng et~al., 2020]{peng2020effects}
Peng, L., Guo, X., and Meyerhoefer, C.~D. (2020).
\newblock The effects of medicaid expansion on labor market outcomes: evidence
  from border counties.
\newblock {\em Health economics}, 29(3):245--260.

\bibitem[Peyr{\'e} and Cuturi, 2019]{peyre2019computational}
Peyr{\'e}, G. and Cuturi, M. (2019).
\newblock Computational optimal transport.
\newblock {\em Foundations and Trends\textsuperscript{\textregistered} in
  Machine Learning}, 11(5-6):355--607.

\bibitem[Pooladian and Niles-Weed, 2021]{pooladian_entropic_2021}
Pooladian, A.-A. and Niles-Weed, J. (2021).
\newblock Entropic estimation of optimal transport maps.
\newblock {\em arXiv:2109.12004 [math, stat]}.
\newblock arXiv: 2109.12004.

\bibitem[Rockafellar, 1970]{rockafellar1970convex}
Rockafellar, R.~T. (1970).
\newblock {\em Convex Analysis}, volume~36.
\newblock Princeton University Press.

\bibitem[Ruggles et~al., 2019]{ruggles2019ipums}
Ruggles, S., Flood, S., Goeken, R., Grover, J., Meyer, E., Pacas, J., and
  Sobek, M. (2019).
\newblock Ipums usa: Version 9.0 [dataset].
\newblock {\em Minneapolis, MN: IPUMS}, 10:D010.

\bibitem[Ruthotto et~al., 2020]{ruthotto2020machine}
Ruthotto, L., Osher, S.~J., Li, W., Nurbekyan, L., and Fung, S.~W. (2020).
\newblock A machine learning framework for solving high-dimensional mean field
  game and mean field control problems.
\newblock {\em Proceedings of the National Academy of Sciences},
  117(17):9183--9193.

\bibitem[Seguy et~al., 2018]{seguy2018large}
Seguy, V., Damodaran, B.~B., Flamary, R., Courty, N., Rolet, A., and Blondel,
  M. (2018).
\newblock Large scale optimal transport and mapping estimation.
\newblock In {\em International Conference on Learning Representations}.

\bibitem[Talagrand, 1992]{talagrand1992matching}
Talagrand, M. (1992).
\newblock Matching random samples in many dimensions.
\newblock {\em The Annals of Applied Probability}, pages 846--856.

\bibitem[Talagrand, 1994]{talagrand1994transportation}
Talagrand, M. (1994).
\newblock The transportation cost from the uniform measure to the empirical
  measure in dimension {$\geq 3$}.
\newblock {\em The Annals of Probability}, pages 919--959.

\bibitem[{van der Vaart} and Wellner, 2013]{wellner2013weak}
{van der Vaart}, A. and Wellner, J. (2013).
\newblock {\em Weak convergence and empirical processes: with applications to
  statistics}.
\newblock Springer Science \& Business Media.

\bibitem[{van der Vaart}, 2000]{van2000asymptotic}
{van der Vaart}, A.~W. (2000).
\newblock {\em Asymptotic statistics}, volume~3.
\newblock Cambridge university press.

\bibitem[Varadarajan, 1958]{varadarajan1958convergence}
Varadarajan, V.~S. (1958).
\newblock On the convergence of sample probability distributions.
\newblock {\em Sankhy{\=a}: The Indian Journal of Statistics}, 19(1):23--26.

\bibitem[Villani, 2003]{villani_topics_2003}
Villani, C. (2003).
\newblock {\em Topics in optimal transportation}.
\newblock Number v. 58 in Graduate studies in mathematics. American
  Mathematical Society, Providence, RI.

\bibitem[Villani, 2009]{villani_optimal_2009}
Villani, C. (2009).
\newblock {\em Optimal transport: old and new}.
\newblock Number 338 in Grundlehren der mathematischen {Wissenschaften}.
  Springer, Berlin Heidelberg.

\bibitem[Weed and Bach, 2019]{weed2019sharp}
Weed, J. and Bach, F. (2019).
\newblock Sharp asymptotic and finite-sample rates of convergence of empirical
  measures in wasserstein distance.
\newblock {\em Bernoulli}, 25(4A):2620--2648.

\bibitem[Werenski et~al., 2022]{werenski2022measure}
Werenski, M.~E., Jiang, R., Tasissa, A., Aeron, S., and Murphy, J.~M. (2022).
\newblock Measure estimation in the barycentric coding model.
\newblock In {\em International Conference on Machine Learning}, pages
  23781--23803. PMLR.

\bibitem[Yuille, 1991]{yuille_deformable_1991}
Yuille, A.~L. (1991).
\newblock Deformable {Templates} for {Face} {Recognition}.
\newblock {\em Journal of Cognitive Neuroscience}, 3(1):59--70.

\bibitem[Zemel and Panaretos, 2019]{zemel2019frechet}
Zemel, Y. and Panaretos, V.~M. (2019).
\newblock Fr{\'e}chet means and procrustes analysis in wasserstein space.
\newblock {\em Bernoulli}, 25(2):932--976.

\end{thebibliography}

\newpage

\appendix 

\section{Proofs from the main text}

\begin{proof}[Proof of \propref{Prop:Solution_Characterization}]
   Define the following closed and convex subset $\mathcal{C}\subset L^2(P_0)$ for fixed optimal transportation maps between $P_0$ and $P_j$, denoted $\nabla\varphi_j$:
   \begin{align*}
       \mathcal{C} \coloneqq \condset{f \in L^2 (P_0)}{f = \sum_{j=1}^{J} \lambda_j \grad \phi_j}{ \text{ for some } }{ \lambda \in \Delta^{J}}.
   \end{align*}
   Recall that the transport maps $\nabla\varphi_j$ exist since $P_0$ is regular.
    Using $\mathcal{C}$, we can rewrite \eqref{Eq:Barycenter_Solving_Simple_Method} as
   \begin{align*}
       \underset{\lambda \in \Delta^{J}}{\argmin} \norm{ \sum_{j=1}^{J} \lambda_j \grad \phi_j - \id }_{L^2 (P_0)}^2 &=
       \underset{f \in \mathcal{C}}{\argmin} \norm{ f - \id }_{L^2 (P_0)}^2,
   \end{align*}
   which by definition is the metric projection of $\id$ onto $\mathcal{C}$. Since $\mathcal{C}$ is a non-empty closed and convex subset of the Hilbert space $L^2(P_0)$, this metric projection exists and is unique \citep[Theorem 6.53]{aliprantis_infinite_1999}. Moreover, if $\id \in \mathcal{C}$, then $\pi_{\mathcal{C}} = \id$; otherwise, $\pi_{\mathcal{C}} \in \partial \mathcal{C}$, where $\partial\mathcal{C}$ is the boundary of $\mathcal{C}$ \citep[Lemma 6.54]{aliprantis_infinite_1999}. 
   
   Since $P_0$ is regular, the exponential map is continuous. In fact, for every $j \neq k$,
   \begin{align*}
       W_2^2(P_j, P_{k}) = W_2^2((\grad \phi_j)_{\#} P_0, (\grad \phi_{k})_{\#} P_0) &\leq \int_{\mathbb{R}^d} \abs{\grad \phi_j - \grad \phi_{k}}^2 \dif P_0 (x).
   \end{align*}
   In other words, the distance between $P_j$ and $P_{k}$ in $\wasserstein_2 (\mathbb{R}^d)$ is smaller than that between corresponding elements $\nabla\varphi_j,\nabla\varphi_k$ in the tangent space. This implies continuity of the exponential map.
   
   Furthermore, in this regular setting, the exponential map sends convex sets in $\tangent_{P_0} \wasserstein_2$ to generalized geodesically convex sets in $\wasserstein_2$. Mechanically, for any two (scaled) elements $t(\nabla\varphi_j-\id)$ and $s(\nabla\varphi_k-\id)$ in $\mathcal{T}_{P_0}\wasserstein_2$, and any $\rho \in [0,1]$,
   \begin{align*}
       &\exp_{P_0} (\rho t ( \grad \phi_j - \id) + (1-\rho) s ( \grad \phi_{k} - \id)) \\
       =& \exp_{P_0} ( (\rho t \grad \phi_j + (1-\rho) s \grad \phi_{k}) - (\rho t + (1-\rho) s) \id) \\
       =& \exp_{P_0} \left( \widetilde{\rho} \left[ \left[ \frac{\rho t}{\widetilde{\rho}} \grad \phi_j + \frac{(1-\rho) s}{\widetilde{\rho}} \grad \phi_{k}  \right] - \id \right] \right) \\
       =& \left( \left[ \rho t \grad \phi_j + (1-\rho) s \grad \phi_{k} \right] + (1-\widetilde{\rho}) \id \right)_{\#} P_0 \\
       =& \left( \left[ \rho t ( \grad \phi_j - \id) + (1-\rho) s ( \grad \phi_{k} - \id) \right] + \id \right)_{\#} P_0
   \end{align*}
   where $\widetilde{\rho} \coloneqq \rho t + (1-\rho) s$. This is a generalized geodesic connecting $P_j$ and $P_{k}$, via the optimal transport map between them and $P_0$ \citep[section 9.2]{ambrosio_gradient_2008}. The same argument holds when extending generalized geodesics to generalized barycenters by taking convex combination of more measures than a binary interpolation with respect to $\rho$. Mechanically, for any $\lambda \in \Delta^{J}$ and $t_j > 0$ for all $j \in \llbracket J \rrbracket$,
   \begin{align*}
       &\exp_{P_0} \left( \sum_{j=1}^{J} \lambda_j t_j ( \grad \phi_j - \id) \right) = \exp_{P_0} \left( \sum_{j=1}^{J} \lambda_j t_j \grad \phi_j - \sum_{j=1}^{J} \lambda_j t_j \id \right) \\
       =& \exp_{P_0} \left( \widetilde{\rho}_J \left[ \sum_{j=1}^{J} \widetilde{\rho}_J \nabla \phi_j - \id \right] \right) \\
       =& \left( \left[ \sum_{j=1}^{J} \lambda_j t_j \nabla \phi_j \right] + (1 - \widetilde{\rho}_J) \id \right)_{\#} P_0 \\
       =& \left( \left[ \sum_{j=1}^{J} \lambda_j t_j (\nabla \phi_j - \id) \right] + \id \right)_{\#} P_0
   \end{align*}
   where $\widetilde{\rho}_J \coloneqq \sum_{j=1}^{J} \lambda_j t_j$. This proves the exponential map is generalized geodesically convex.
   

   From above it follows that $P_{\pi} \coloneqq \exp_{P_0} (\pi_C)$ is either in the interior of $\mathcal{C}$, which is the case if $\id\in \mathcal{C}$, or on its boundary: since $\pi_C \in \partial C$, $\exp_{P_0} (\pi_C) \in \exp_{P_0} (\partial C)$. By continuity of the exponential map it follows that $\exp_{P_0} (\partial C) = \partial \exp_{P_0} (C)$. 
   Combining all steps above show that $P_{\pi}$ is a \textit{geodesic} metric projection of $P_0$ onto the geodesic convex hull of $\set{P_j}_{j=1}^{J}$.
\end{proof}

\begin{proof}[Proof of \propref{Prop:Solution_Characterization_general}]
    The result follows from the same argument as the proof of Proposition \ref{Prop:Solution_Characterization}. Theorem 12.4.4 in \citet{ambrosio_gradient_2008} shows that $\mathcal{T}_{P_0}\wasserstein_2$ is the image of the barycentric projection of measures in the general tangent cone: $b_\gamma(x)$ is an optimal transport map if $\gamma$ is an optimal transport plan. But the exponential map satisfies 
    \[\exp_{P_0}(v) = \left(v+\id\right)_\# P_0\qquad\text{for all $v\in\mathcal{T}_{P_0}\wasserstein_2$.}\] This implies that \[\widetilde{P}_\pi\coloneqq \exp_{P_0}\left(\sum_{j=1}^J \lambda_j^*b_{\gamma_{0j}}-\id\right) = \left(\sum_{j=1}^J \lambda_j^*b_{\gamma_{0j}}\right)_{\raisebox{12pt}{$\scaleobj{0.75}{\#}$}} P_0\in \widetilde{\convexhull}_{P_0}\left(\left\{P_j\right\}_{j=1}^J\right) ~, \]
    since the convex combination of elements in the subgradients of convex functions lie in the subgradient of a convex function (provided the subgradient of each convex function is nonempty, which is the case here). Then the continuity and generalized convexity of the exponential map for elements in the regular tangent space $\mathcal{T}_{P_0}\wasserstein_2$ implies the result.
\end{proof}

\begin{proof}[Proof of Proposition \ref{prop:consistency}]
We split the proof into two parts. In the first part we prove the convergence in probability of the family of objective functions \eqref{Eq:Barycenter_Solving_gen_emp} to their population counterparts \eqref{Eq:Barycenter_Solving_gen} if the empirical measures $\prob_{N_j}$ converge weakly in probability to the population measures $P_j$. 
In the second step we use the fact that $\widehat{\lambda}^*$ is a classical semiparametric estimator  \citep{andrews1994asymptotics,newey1994large} to derive the convergence of the weights. \\
%

\paragraph{\emph{Step 1: Convergence of the objective functions}}
To show the convergence of the of the objective functions for obtaining the weights $\lambda^*$, we write
\begin{align*}
   & \left\lvert \left\|\sum_{j=1}^J\lambda_jb_{0j}-\id\right\|^2_{L^2(P_0)} - \left\|\sum_{j=1}^J\lambda_j\widehat{b}_{0j}-\id\right\|^2_{L^2(\prob_{N_0})}\right\rvert\\
   =& \left\lvert \int \left\lvert \sum_{j=1}^J\lambda_jb_{0j}(x)-x \right\rvert^2 \dif P_0 - \int \left\lvert \sum_{j=1}^J\lambda_j\widehat{b}_{0j}(x)-x \right\rvert^2 \dif \prob_{N_0}\right\rvert.
\end{align*}
We hence want to show that
\[\lim_{\bigwedge_j N_j\to \infty} \left\lvert \int \left\lvert \sum_{j=1}^J\lambda_jb_{0j}(x)-x \right\rvert^2 \dif P_0(x) - \int \left\lvert \sum_{j=1}^J\lambda_j\widehat{b}_{0j}(x)-x \right\rvert^2 \dif \prob_{N_0}(x)\right\rvert = 0 ~, \] where 
$\bigwedge_j N_j \equiv \min\left\{N_0,\ldots,N_J\right\}$.

We split the result into two parts. The first part shows that 
\[\liminf_{\bigwedge_jN_j\to\infty}\int_{\mathbb{R}^d} \left\lvert \sum_{j=1}^J\lambda_j\widehat{b}_{0j}(x_0)-x_0 \right\rvert^2\dif\prob_{N_0}(x_0) \geq \int_{\mathbb{R}^d} \left\lvert \sum_{j=1}^J\lambda_jb_{0j}(x_0)-x_0 \right\rvert^2\dif P_0(x_0).\]
In the second part we use the $L^2(P_0)$ convergence of the barycentric projections to prove that the limit exists and coincides with the limit inferior.

For the first part, we have 
\begin{multline*}
    \liminf_{\bigwedge_jN_j\to\infty}\int_{\mathbb{R}^d} \left\lvert \sum_{j=1}^J\lambda_j\widehat{b}_{0j}(x_0)-x_0 \right\rvert^2\dif\prob_{N_0}(x_0)= 
    \liminf_{\bigwedge_jN_j\to\infty}\int_{(\mathbb{R}^d)^{J+1}} \left\lvert \sum_{j=1}^J\lambda_jx_j-x_0 \right\rvert^2\dif\bm{\widehat{\gamma}}_N(x_0,x_1,\ldots,x_J),
\end{multline*}
where $\widehat{\bm{\gamma}}_N(x_0,x_1,\ldots,x_J)$ is a measure that solves
\[\min\left\{\int_{(\mathbb{R}^d)^{J+1}} \sum_{j=1}^J\lambda_j \left\lvert x_j-x_0\right\rvert^2\dif \bm{\gamma}: \bm{\gamma}\in\Gamma_1(\widehat{\gamma}_{01},\ldots,\widehat{\gamma}_{0J})\right\} ~, \]
$\widehat{\gamma}_{0j}$ are the optimal couplings between $\prob_{N_0}$ and $\widetilde{\prob}_{N_j}\coloneqq \left(\widehat{b}_{0j}\right)_\#\prob_{N_0}$. Since all measures are defined on the complete and separable space $\mathbb{R}^d$, and by assumption of finite second moments, i.e.
\[\max_{j\in\llbracket J\rrbracket}\sup_{N_j} \int \left\lvert x_j-x_0 \right\rvert^2\dif\widehat{\gamma}_{0j} <+\infty ~, \] it holds that each sequence $\widehat{\gamma}_{0j}$ is tight by Ulam's theorem \citep[Theorem 7.1.4]{dudley2018real}. Using the fact that $\lambda\in\Delta^{J}$ and $\widehat{\bm{\gamma}}_N\in\Gamma_1\left(\widehat{\gamma}_{01},\ldots,\widehat{\gamma}_{0J}\right)$, applying Jensen's inequality gives us
\[\max_{j\in\llbracket J\rrbracket} \sup_{N_j} \int_{\left(R^d\right)^{J+1}} \left\lvert\sum_{j=1}^J\lambda_j x_j-x_0\right\rvert^2\dif\widehat{\bm{\gamma}}_N \leq \max_{j\in\llbracket J\rrbracket} \sup_{N_j}\sum_{j=1}^J\lambda_j\int_{\mathbb{R}^d}\left\lvert x_j-x_0\right\rvert^2 \dif\widehat{\gamma}_{0j}<+\infty ~, \] which implies that $\widehat{\bm{\gamma}}_N$ is tight. 
By Prokhorov's theorem, there exists a subsequence $\widehat{\bm{\gamma}}_{N_k}$ that weakly converges to a limit measure $\bm{\gamma}$. Therefore, by the continuity of the map $(x_0,x_j) \mapsto \sum_{j}\lambda_jx_j-x_0$, it follows from classical convergence results \citep[Lemma 5.1.12(d)]{ambrosio_gradient_2008} that 
\begin{multline*}
    \liminf_{\bigwedge_jN_j\to\infty}\int_{(\mathbb{R}^d)^{J+1}} \left\lvert \sum_{j=1}^J\lambda_jx_j-x_0 \right\rvert^2\dif\bm{\widehat{\gamma}}_N(x_0,x_1,\ldots,x_J)=
     \int_{(\mathbb{R}^d)^{J+1}} \left\lvert \sum_{j=1}^J\lambda_jx_j-x_0 \right\rvert^2\dif\bm{\gamma}(x_0,\ldots,x_J).
\end{multline*}
Furthermore, by the same argument via Jensen's inequality, i.e., 
\[ \int_{(\mathbb{R}^d)^{J+1}} \left\lvert \sum_{j=1}^J\lambda_jx_j-x_0 \right\rvert^2\dif\bm{\gamma}(x_0,\ldots,x_J) \leq \sum_{j=1}^J \int_{(\mathbb{R}^d)^{2}}  \left\lvert\lambda_jx_j-x_0 \right\rvert^2\dif\gamma_{0j}(x_0,x_j) < + \infty ~, \] it follows that the limit $\bm{\gamma}\in \Gamma_1\left(\gamma_{01},\ldots,\gamma_{0J}\right)$.

Now note that by the definition of disintegration it follows that \citep[Lemma 5.3.2]{ambrosio_gradient_2008} \[\bm{\gamma}\in\Gamma_1(\gamma_{01},\ldots,\gamma_{0J})\qquad \Longleftrightarrow\qquad \bm{\gamma}_{x_0}\in \Gamma\left(\gamma_{1|x_0},\ldots,\gamma_{J|x_0}\right) ~, \]
where 
\[\bm{\gamma} = \int \bm{\gamma}_{x_0}\dif P_0(x_0)\qquad\text{and}\qquad \gamma_{0j} = \int \gamma_{j|x_0}\dif P_0(x_0)\] are the disintegrations of $\bm{\gamma}$ and $\gamma_{0j}$ with respect to $P_0$, respectively. Therefore, we have
\begin{align*}
    & \int_{(\mathbb{R}^d)^{J+1}} \left\lvert \sum_{j=1}^J\lambda_jx_j-x_0 \right\rvert^2\dif\bm{\gamma}(x_0,\ldots,x_J)\\
    =& \int_{\mathbb{R}^d}\int_{\left(\mathbb{R}^d\right)^J} \left\lvert \sum_{j=1}^J\lambda_jx_j-x_0 \right\rvert^2\dif\bm{\gamma}_{x_0}(x_1,\ldots,x_J)\dif P_0(x_0)\\
    \geq&\int_{\mathbb{R}^d}\left\lvert\int_{\left(\mathbb{R}^d\right)^J}\left(\sum_{j=1}^J\lambda_j x_j-x_0\right)\dif\bm{\gamma}_{x_0}(x_1,\ldots,x_J)\right\rvert^2\dif P_0(x_0)\\
    =&\int_{\mathbb{R}^d}\left\lvert \sum_{j=1}^J\lambda_j\int_{\left(\mathbb{R}^d\right)^J}x_j\dif\bm{\gamma}_{x_0}(x_1,\ldots,x_J)-x_0\right\rvert^2\dif P_0(x_0)\\
    =&\int_{\mathbb{R}^d}\left\lvert\sum_{j=1}^J\lambda_j\int_{\mathbb{R}^d}x_j\dif\gamma_{j|x_0}(x_j)-x_0\right\rvert^2\dif P_0(x_0)\\
    =&\int_{\mathbb{R}^d}\left\lvert\sum_{j=1}^J\lambda_jb_{0j}(x_0) - x_0\right\rvert^2\dif P_0(x_0),
\end{align*}
where the third lines follows from Jensen's inequality and the fifth line from $\bm{\gamma}_{x_0}\in \Gamma\left(\gamma_{1|x_0},\ldots,\gamma_{J|x_0}\right)$. This shows the first part.

For the second part we use the fact that each barycentric projection $\widehat{b}_{0j}(x_1)$ is an optimal transport map between $\prob_{N_0}$ and $\widetilde{\prob}_{N_j}$ if $\widehat{\gamma}_{0j}$ is an optimal transport plan between $\prob_{N_0}$ and $\prob_{N_j}$, which follows from Theorem 12.4.4 in \citet{ambrosio_gradient_2008}.
As before, we know that $\left(\widehat{b}_{0j}\right)_\#\prob_{N_0}$ is a tight sequence that converges to some $\widetilde{P}_j$. By definition and the fact that $\hat{b}_{0j}$ is the gradient of a convex function between $\prob_{N_0}$ and $\widetilde{\prob}_{N_j}$, $\widehat{b}_{0j}$ is the unique optimal transport map between $\prob_{N_0}$ and $\widetilde{\prob}_{N_j}$ for all $N_j$ and all $j$. Since the measures $P_j$ have finite second moments by assumption, we have
\begin{align*}
    \limsup_{N_0\wedge N_j\to\infty} \int_{\mathbb{R}^d} |x_j|^2\dif \widetilde{\prob}_{N_j} =& \limsup_{N_0\wedge N_j\to\infty} \int_{\mathbb{R}^d} \left\lvert \widehat{b}_{0j}(x_0)\right\rvert^2\dif \prob_{N_0} \\
    =& \limsup_{N_0\wedge N_j\to\infty} \int_{\mathbb{R}^d} \left\lvert \int_{\mathbb{R}^d} x_j\dif \widehat{\gamma}_{j|x_0}(x_j)\right\rvert^2\dif \prob_{N_0}\\
    \leq &\limsup_{N_0\wedge N_j\to\infty}\int_{\left(\mathbb{R}^d\right)^2}\left\lvert x_j\right\rvert^2 \dif \widehat{\gamma}_{0j}(x_0,x_j)\\
    = & \int_{\left(\mathbb{R}^d\right)^2}\left\lvert x_j\right\rvert^2 \dif \gamma_{0j}(x_0,x_j)<+\infty,
\end{align*}
where the last equality follows from the tightness of $\widehat{\gamma}_{0j}$, as shown earlier.
Therefore, by standard stability results for optimal transport maps, in particular Proposition 6 in \citet{zemel2019frechet}, it holds that
$\widehat{b}_{0j}$ converges uniformly on every compact subset $K\subset\mathbb{R}^d$ in the interior of the support of the limit measure $\widetilde{P}_j$, that is
\[\lim_{N_0\wedge N_j\to\infty}\sup_{x_0\in K}\left\lvert\widehat{b}_{0j}(x_0) - v_j(x_0)\right\rvert = 0 ~, \]
where $v_j$ is the optimal transport map between $P_0$ and $\widetilde{P}_j$.
%
%

We now show that $v_j=b_{0j}$ $P_0$-almost everywhere. 
From the local uniform convergence, we can then derive ``strong $L^2$-convergence'' \citep[Definition 5.4.3]{ambrosio_gradient_2008} of the potentials:
\begin{align*}
    &\limsup_{N_0\wedge N_j\to\infty}\left\lvert\left\|\widehat{b}_{0j}\right\|_{L^2(\prob_{N_0})} - \left\|v_j\right\|_{L^2(P_{0})}\right\rvert\\
    \leq & \limsup_{N_0\wedge N_j\to\infty}\left\lvert\left\|\widehat{b}_{0j}\right\|_{L^2(\prob_{N_0})} - \left\|v_j\right\|_{L^2(\prob_{N_0})}\right\rvert + \limsup_{N_0\to\infty}\left\lvert\left\|v_j\right\|_{L^2(\prob_{N_0})} - \left\|v_j\right\|_{L^2(P_{0})}\right\rvert\\
    \leq & \limsup_{N_0\wedge N_j\to\infty}\left\|\widehat{b}_{0j} - v_j\right\|_{L^2(\prob_{N_0})} + \limsup_{N_0\to\infty}\left\lvert\left\|v_j\right\|_{L^2(\prob_{N_0})} - \left\|v_j\right\|_{L^2(P_{0})}\right\rvert
\end{align*}
Now the first term converges to zero by H\"older's inequality and the local uniform convergence of the optimal transport maps from above. The second term satisfies
\begin{align*}
    &\limsup_{N_0\to\infty}\left\lvert\left\|v_j\right\|_{L^2(\prob_{N_0})} - \left\|v_j\right\|_{L^2(P_{0})}\right\rvert \\
    =& \limsup_{N_0\to\infty}\left\lvert \left(\int_{\mathbb{R}^d} \left\lvert v_j(x_0)\right\rvert^2 \dif \prob_{N_0}\right)^{1/2} - \left(\int_{\mathbb{R}^d} \left\lvert v_j(x_0)\right\rvert^2 \dif P_{0}\right)^{1/2}\right\rvert\\
    \leq & \limsup_{N_0\to\infty}\left\lvert \int_{\mathbb{R}^d} \left\lvert v_j(x_0)\right\rvert^2 \dif \prob_{N_0} - \int_{\mathbb{R}^d} \left\lvert v_j(x_0)\right\rvert^2 \dif P_{0}\right\rvert^{1/2}.
\end{align*}
But since $P_0$ has finite second moments, it holds that this term also converges to zero.

Based on this we can show that $\widehat{\gamma}_{0j}\equiv \left(\id,\widehat{b}_{0j}\right)$ converge weakly to $\gamma_{0j}\equiv \left(\id,v_j\right)$. Indeed, if $\gamma_{0j}$ is a limit point of the sequence $\widehat{\gamma}_{0j}$, it holds that
\begin{align*}
    \int_{\left(\mathbb{R}^d\right)^2} \left\lvert x_j\right\rvert^2\dif\gamma_{0j}(x_0,x_j) &\leq \liminf_{N_0\wedge N_j\to\infty} \int_{\left(\mathbb{R}^d\right)^2}\left\lvert x_j\right\rvert^2\dif\widehat{\gamma}_{0j}(x_0,x_j)\\
    & \leq \limsup_{N_0\wedge N_j\to\infty} \int_{\left(\mathbb{R}^d\right)^2}\left\lvert x_j\right\rvert^2\dif\widehat{\gamma}_{0j}(x_0,x_j)\\
    & = \limsup_{N_0\wedge N_j\to\infty} \int_{\mathbb{R}^d}\left\lvert \widehat{b}_{0j}(x_0)\right\rvert^2\dif\prob_{N_0}(x_0)\\
    &=\int_{\mathbb{R}^d} \left\lvert v_j(x_0)\right\rvert^2 \dif P_0(x_0).
\end{align*}
Disintegrating the left-hand side with respect to $P_0$, and applying Jensen's inequality, gives
\begin{align*}
    \int_{\left(\mathbb{R}^d\right)^2} \left\lvert x_j\right\rvert^2\dif\gamma_{0j}(x_0,x_j) &= \int_{\mathbb{R}^d}\int_{\mathbb{R}^d} \left\lvert x_j\right\rvert^2\dif\gamma_{j|x_0}(x_j)\dif P_0(x_0)\\
    &\geq \int_{\mathbb{R}^d}\left\lvert\int_{\mathbb{R}^d}  x_j\dif\gamma_{j|x_0}(x_j)\right\rvert^2\dif P_0(x_0)\\
    & = \int_{\mathbb{R}^d} \left\lvert b_{0j}(x_0)\right\rvert^2 \dif P_0(x_0),
\end{align*}
that is,
\[\int_{\mathbb{R}^d} \left\lvert b_{0j}(x_0)\right\rvert^2 \dif P_0(x_0)\leq \int_{\mathbb{R}^d} \left\lvert v_j(x_0)\right\rvert^2 \dif P_0(x_0).\] But since $v_j$ is an optimal transport map between $P_0$ and $\widetilde{P}_j$ by definition, it holds that
\[
    \int_{\mathbb{R}^d} \left\lvert b_{0j}(x_0)\right\rvert^2 \dif P_0(x_0)\geq \int_{\mathbb{R}^d} \left\lvert v_j(x_0)\right\rvert^2 \dif P_0(x_0) ~,
\]
which implies that equality holds and we have that
\[\int_{\mathbb{R}^d}\left[\left\lvert b_{0j}(x_0)\right\rvert^2 - \left\lvert v_j(x_0)\right\rvert^2\right] \dif P_0(x_0) = 0 ~, \]
which implies that $b_{0j} = v_j$ $P_0$-almost everywhere. We have hence shown that $\left(\id,\widehat{b}_{0j}\right)_\#\prob_{N_0}$ converges weakly to $\left(\id,b_{0j}\right)_\#P_0$ for all $j$, where the barycentric projection $b_{0j}$ is the optimal transport map between $P_0$ and $\widetilde{P}_j$ \citep[e.g.][Theorem 2.12.(iii)]{villani_topics_2003}.

Moreover, we have shown ``strong $L^2$-convergence'' of the barycentric projections in terms of Definition 5.4.3 in \citet{ambrosio_gradient_2008}. Since this holds for all $j$, it also holds for their convex combination for fixed weights $\lambda\in\Delta^{J}$. Putting everything together, we then have that
\[\lim_{\bigwedge_j N_j\to\infty}\left\|\sum_{j=1}^J\lambda_j\widehat{b}_{0j}-\id\right\|^2_{L^2(\prob_{N_0})} = \left\|\sum_{j=1}^J\lambda_j b_{0j}-\id\right\|^2_{L^2(P_0)}.\] Since all observable measures $\prob_j$ are empirical measures, they converge weakly in probability \citep{varadarajan1958convergence}, which implies that
\[\lim_{\bigwedge_jN_j\to\infty} P\left(\left\lvert\left\|\sum_{j=1}^J\lambda_j\widehat{b}_{0j}-\id\right\|^2_{L^2(\prob_{N_0})} - \left\|\sum_{j=1}^J\lambda_j b_{0j}-\id\right\|^2_{L^2(P_0)}\right\rvert>\varepsilon\right) = 0\qquad\text{for all $\varepsilon>0$}.\] This shows convergence in probability of the objective function for fixed $\lambda$. \\

\paragraph{\emph{Step 2: Convergence of the optimal weights $\widehat{\lambda}_N^*$}} The convergence of the optimal weights now follows from standard consistency results in semiparametric estimation. In particular, the objective functions are all convex for any $\lambda\in\mathbb{R}^J$, which implies that they converge uniformly on any compact set \citep[Theorem 10.8]{rockafellar1970convex}, so the objective function converges uniformly on $\Delta^{J}$. Now a standard consistency result like Theorem 2.1 in \citet{newey1994large} then implies that 
\[\lim_{\bigwedge_j N_j\to\infty}P\left(\left\lvert \widehat{\lambda}^*_N - \lambda^*\right\rvert>\varepsilon\right) = 0\qquad\text{for all $\varepsilon>0$} ~, \] which is what we wanted to show. Note that the result can also be shown if we allow the weights $\lambda$ to be negative, i.e., if we only require that $\sum_{j=1}^J\lambda_j = 1$. In this case, the fact that the objective functions are convex and coercive implies that an optimal $\lambda^*$ will be achieved at the interior of the extended Euclidean space, from which consistency follows by Theorem 2.7 in \citet{newey1994large}.
\end{proof}

\begin{proof}[Proof of Corollary \ref{cor:consistency}]
We want to show that $\left(\sum_{j=1}^J\hat{\lambda}_{N_j}^*\hat{b}_{0j}\right)_\#\prob_{N_0}$ converges weakly in probability to $\left(\sum_{j=1}^J\lambda_{j}^*b_{0j}\right)_\# P_{0}$, where $\hat{\lambda}_N^*\coloneqq \left(\hat{\lambda}_{N_1}^*,\ldots,\hat{\lambda}_{N_J}^*\right)$ are the optimal weights obtained in \eqref{Eq:Barycenter_Solving_gen_emp} and \eqref{Eq:Barycenter_Solving_gen}, respectively. The result follows by applying the extended continuous mapping theorem \citep[Theorem 1.11.1]{wellner2013weak} as follows.

As shown in the proof of Proposition \ref{prop:consistency} we have ``strong $L^2$-convergence'' of the maps $\sum_{j=1}^J\hat{\lambda}^*_{N_j}\hat{b}_{0j}-\id$ to $\sum_{j=1}^J\lambda^*_jb_{0j}-\id$. Therefore, by Theorem 5.4.4 (iii) in \citet{ambrosio_gradient_2008}, it holds that
\[\lim_{\wedge_j N_j\to\infty} \int_{\mathbb{R}^d} f\left(x_0,\sum_{j=1}^J\hat{\lambda}_{N_j}^*\hat{b}_{0j}(x_0)-x_0\right) \dif\prob_{N_0}(x_0) = \int_{\mathbb{R}^d} f\left(x_0,\sum_{j=1}^J\lambda_{j}^*b_{0j}(x_0)-x_0\right) \dif P_{0}(x_0)\] for any continuous function such that $|f(x_0)|\leq C_1 + C_2\left\lvert \overline{x}_0 - x_0\right\rvert^2$ for all $x_0$ in the support of $P_0$, where $C_1,C_2<+\infty$ are some constants and $\overline{x}_0$ in some element in the support of $P_0$ \citep[equation (5.1.21)]{ambrosio_gradient_2008}.
%
%
In particular, this holds for any bounded and continuous function $f$, which implies that
\[\lim_{\wedge_j N_j\to\infty} \int_{\mathbb{R}^d} f\left(\sum_{j=1}^J\hat{\lambda}_{N_j}^*\hat{b}_{0j}(x_0)\right) \dif\prob_{N_0}(x_0) = \int_{\mathbb{R}^d} f\left(\sum_{j=1}^J\lambda_{j}^*b_{0j}(x_0)\right) \dif P_{0}(x_0)\] for any bounded and continuous function, which implies that $\left(\sum_{j=1}^J \hat{\lambda}^*_{N_j}\hat{b}_{0j}\right)_\#\prob_{N_0}$ converges weakly to $\left(\sum_{j=1}^J \lambda^*_{j}b_{0j}\right)_\# P_{0}$ if $\prob_{N_j}$ converge weakly to $P_j$, $j\in\llbracket J\rrbracket$.

Now we apply the extended continuous mapping theorem \citep[Theorem 1.11.1]{wellner2013weak}. Equip $\mathscr{P}_2(\mathbb{R}^d)$ with any metric $\widetilde{d}(\cdot,\cdot)$ that metrizes weak convergence.
%
%
We define the maps $g: \bigtimes_{j=0}^J\left(\mathscr{P}_2(\mathbb{R}^d), \widetilde{d}\right)_j\to \left(\mathscr{P}_2(\mathbb{R}^d), \widetilde{d}\right)$ by 
\[g\left(P_0,\ldots,P_J\right) = \left(\sum_{j=1}^J\lambda_j^*b_{0j}\right)_\# P_0 ~, \] and analogously for their empirical counterparts $g_N$. Note that $g$ and $g_N$ are non-random functions if the measures $P_j$ and $\prob_{N_j}$ are non-random themselves for all $j \in \llbracket J \rrbracket$. Moreover, by definition, $g$ and $g_N$ are continuous maps because $\sum_{j=1}^J\lambda_j^* b_{0j}$ are gradients of convex functions, which are continuous $P_0$-almost everywhere; the same thing holds for their empirical counterparts. Now from what we have shown above and in Proposition \ref{prop:consistency}, it holds that 
\[g_N\left(\prob_{N_0},\ldots,\prob_{N_J}\right) \to g\left(P_0,\ldots, P_J\right)\] as $\prob_{N_j}$ converge weakly to $P_j$. Since $\{ \prob_{N_j} \}_{j=1}^{J}$ here instead are the only random elements in $\bigtimes_{j=0}^J\left(\mathscr{P}_2(\mathbb{R}^d), \widetilde{d}\right)_j$, the extended continuous mapping theorem implies that
\[
    \lim_{\bigwedge_jN_j\to\infty} P\left(\widetilde{d}\left(g_N\left(\prob_{N_0},\ldots,\prob_{N_J}\right), g\left(P_0,\ldots, P_J\right)\right)>\varepsilon\right)=0\qquad\text{for all $\varepsilon>0$} ~,
\]
which is what we wanted to show.
%
\end{proof}


\section{Details from Simulations and Applications}
\label{Appendix:additional-details-application}

Our implementation is available at the GitHub repository \href{https://github.com/menghsuanhsieh/tangential-wasserstein-projection}{here}. The images used can be found in the repository; the Medicaid data can be downloaded from the Dropbox folder \href{https://www.dropbox.com/sh/y43s568l44ny8pz/AADmFeUdanq9PzKHYhESaPBaa?dl=0}{here}.

\subsection{Mixtures of Gaussian Simulation}
\label{Appendix:mixtures}

As a minor extension from our Gaussian simulation in the main text, we consider mixtures of Gaussian in dimension $d = 20$. As before, we draw from the following Gaussians:
    \[\mathbf{X}_j \sim \Normal \left( \mu_j, \Sigma \right), \quad j=0,1,2,3 ~, \]
where $\mu_0 = [10, 10, \dots, 10]$, $\mu_1 = [50, 50, \dots, 50]$, $\mu_2 = [200, 200, \dots, 200]$, $\mu_3 = [-50, -50, \dots, -50]$ and $\Sigma = \id_{20}+ 0.8 \id^{-}_{10}$, with $\id_{10}^{-}$ the $10\times 10$ matrix with zeros on the main diagonal and ones on all off-diagonal terms. We then define the following mixtures: $\mathbf{Y}_0$ as target, and $\mathbf{Y}_1$, $\mathbf{Y}_2$, and $\mathbf{Y}_3$ as controls, where
\begin{align*}
    \mathbf{Y}_0 &= 0.3 \mathbf{X}_0 + 0.6 \mathbf{X}_1 + 0.1\mathbf{X}_2\\
    \mathbf{Y}_1 &= 0.8 \mathbf{X}_0 + 0.1 \mathbf{X}_1 + 0.1\mathbf{X}_2 \\
    \mathbf{Y}_2 &= 0.2 \mathbf{X}_1 + 0.7 \mathbf{X}_2 + 0.1\mathbf{X}_3 \\
    \mathbf{Y}_3 &= 0.2 \mathbf{X}_0 + 0.2 \mathbf{X}_2 + 0.6\mathbf{X}_3
\end{align*}
We sample $N_0 = N_1 = N_2 = N_3 = 10000$ points from each mixture. The estimated optimal weights are $\lambda^* = [0.8204, 0.1796, 0]$. The results are as expected, since $\mathbf Y_1$ is defined by the same set of $\mathbf X$'s as those for $\mathbf Y_0$. In the case of $\mathbf Y_3$, it is not defined with $\mathbf X_1$, therefore we expected the associated weight to be very small. Table \ref{tab:mean-mixed-gaussian} suggests the projection is, again, very close to the target distribution when only considering the mean, again despite only having three control units.  


\begin{table}[H]
    \centering
    \footnotesize
    \begin{tabular}{|c|c|c|c|c|c|c|c|c|c|c|}
        \hline
        $d$ & 1 & 2 & 3 & 4 & 5 & 6 & 7 & 8 & 9 & 10 \\
        \hline
        mean$(\mathbf Y_0)$ & 52.994 & 52.991 & 52.998 & 52.998 & 52.998 & 52.997 & 53.007 & 52.995 & 52.992 & 53.006 \\
        \hline
        mean$(\sum_{j=1}^{3} \lambda_j^* \mathbf Y_j)$ & 53.114 & 53.107 & 53.108 & 53.122 & 53.115 & 53.113 & 53.112 & 53.122 & 53.106 & 53.122 \\
        \hline
    \end{tabular}
    \caption{Means of target and optimally-weighted controls $\mathbf Y_1$, $\mathbf Y_2$, $\mathbf Y_3$.}
    \label{tab:mean-mixed-gaussian}
\end{table}


\subsection{Details of Medicaid Expansion Application}
\label{Appendix:Details_Medicaid}

We use the ACS data with harmonized variables made available by IPUMS. The data is at the household-person-year level. For our application, we select the household head and the spouse as our unit of analysis. The continuous outcomes are adjusted using the person-level sample weights available in the data.

We adopt the following sample restriction criteria: we included individuals
\begin{itemize}
    \item of working age, i.e. between ages 18 and 65
    
    \item who has no missing outcomes (for those listed in the main text)
    
    \item who has no top-coded responses
    
    \item who are either household heads or their spouses
\end{itemize}
The sample size breakdown by states are follows:
\begin{table}[H]
    \centering
    \begin{tabular}{lcc}
    \hline\hline
      & \textbf{State} & \textbf{Observations} \\ 
    \hline
    \multicolumn{3}{l}{\textbf{Target}} \\ \hline
      & MT & 25,173 \\ 
    \hline\hline
    \multicolumn{3}{l}{\textbf{Control}} \\ \hline
      & AL & 106,464 \\ 
      & FL & 427,397 \\ 
      & GA & 227,659 \\ 
      & KS & 74,812 \\ 
      & MS & 61,505 \\ 
      & NC & 233,804 \\ 
      & SC & 107,905 \\ 
      & SD & 22,563 \\ 
      & TN & 152,470 \\ 
      & TX & 598,222 \\ 
      & WI & 157,410 \\ 
      & WY & 15,666 \\ 
   \hline\hline 
    \end{tabular}
    \caption{Summary of the full data sample used to obtain $\lambda^*$.}
\end{table}
\noindent We randomly select $N=1500$ observations from each unit for estimating $\lambda^*$. In the Python implementation, we face a challenge where if the entries of the target and control data are large enough, \eqref{Eq:Barycenter_Solving_Simple_Method} becomes too large for \texttt{CVXPY} to compute an optimal solution. Therefore, we introduce a stabilizing constant to prevent this. This stabilizing constant is determined by the mean value and dimensions of the target distribution, and the number of controls.

We check whether the obtained weights are fit for creating synthetic Montana by examining if they well-approximate actual Montana in the pre-treatment period. As seen in Figures \ref{fig:replication-comparison-binary} and \ref{fig:replication-comparison-continuous}, our projection is very similar to the actual data.

\begin{figure}[H]
    \centering
    \begin{subfigure}{0.47\textwidth}
         \centering
         \includegraphics[width=\textwidth]{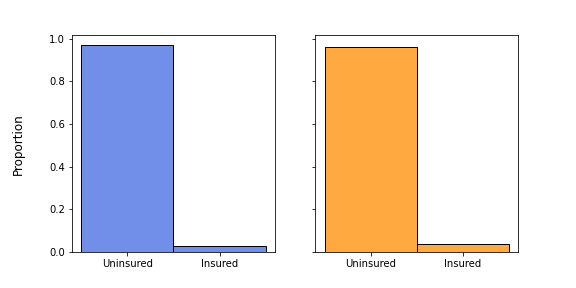}
         \caption{Covered by Medicaid}
         \label{fig:hinscaid-3}
    \end{subfigure}
    \hfill
    \begin{subfigure}{0.47\textwidth}
         \centering
         \includegraphics[width=\textwidth]{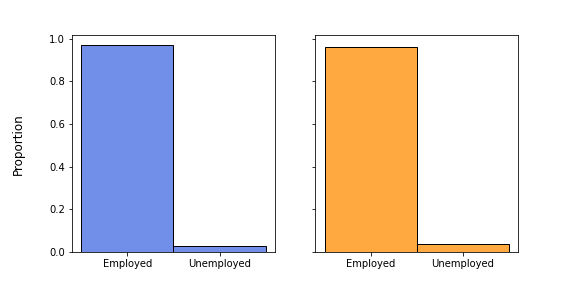}
         \caption{Employment Status}
         \label{fig:empstat-4}
    \end{subfigure}
    \caption{Replicated (blue) vs actual (orange) Montana from 2010 to 2016.}
    \label{fig:replication-comparison-binary}
\end{figure}

\begin{figure}[H]
    \centering
    \begin{subfigure}{0.47\textwidth}
         \centering
         \includegraphics[width=\textwidth]{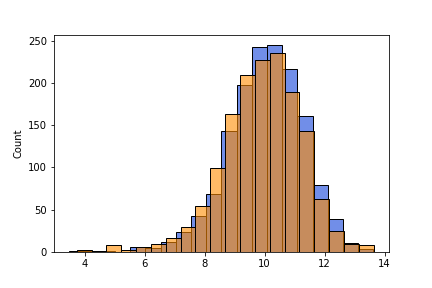}
         \caption{Log Wage}
         \label{fig:hinscaid-4}
    \end{subfigure}
    \hfill
    \begin{subfigure}{0.47\textwidth}
         \centering
         \includegraphics[width=\textwidth]{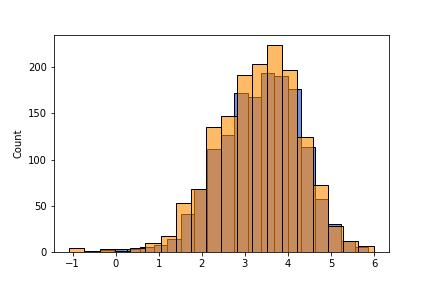}
         \caption{Log Labor Hours Supplied}
         \label{fig:empstat-5}
    \end{subfigure}
    \caption{Replicated (blue) vs actual (orange) Montana from 2010 to 2016. In each panel, histograms of data distributions are shown on the left, and cumulative distribution functions are shown on the right.}
    \label{fig:replication-comparison-continuous}
\end{figure}

Once we obtain the optimal weights $\lambda^*$, we estimate the counterfactual outcomes of interest for the four years after Medicaid expansion in Montana (namely, between 2017 and 2020). In practice, let $\widehat{F}_{t, j} (v)$ denote the empirical distribution function for $v \in \{\rm HINSCAID$, $\rm EMPSTAT$, $\rm UHRSWORK$, $\rm INCWAGE\}$ in state $j$ during the year $t \in \set{2017, 2018, 2019, 2020}$; then, the counterfactual distribution of outcome $v$ for Montana in year $t$ is defined as
\[
    \widehat{F}_{t, \rm Montana}^{\rm \: cf} (v) \triangleq \sum_{j \in \rm \{control \: states\}} \lambda^*_j \widehat{F}_{t,j} (v)
\]
We plot the densities and distributions of the counterfactual outcomes in Figures \ref{fig:first-order-results} and \ref{fig:second-order-results} of the main text.

\end{document}